\documentclass{article} 
\usepackage{iclr2026_conference,times}

\usepackage{amsmath,amsfonts,bm}

\def\figref#1{figure~\ref{#1}}

\def\secref#1{section~\ref{#1}}

\def\eqref#1{equation~\ref{#1}}

\def\1{\bm{1}}

\DeclareMathAlphabet{\mathsfit}{\encodingdefault}{\sfdefault}{m}{sl}
\SetMathAlphabet{\mathsfit}{bold}{\encodingdefault}{\sfdefault}{bx}{n}

\newcommand{\Var}{\mathrm{Var}}

\usepackage{amsmath,amsthm,amssymb}
\usepackage{geometry}
\usepackage{booktabs}
\usepackage{multirow}
\usepackage{graphicx}
\usepackage{enumitem}
\usepackage{caption}
\usepackage{subcaption}
\usepackage{mathtools}  

\usepackage{titletoc} 

\usepackage{hyperref}  
\usepackage{url}

\newtheorem{theorem}{Theorem}

\newtheorem{remark}{Remark}

\newtheorem{lemma}{Lemma}

\newcommand{\ie}{\textit{i}.\textit{e}.}
\newcommand{\eg}{\textit{e}.\textit{g}.}
\newcommand{\norm}[1]{\left\lVert #1 \right\rVert}

\def\secref#1{Section~\ref{#1}}
\def\appref#1{Appendix~\ref{#1}}
\def\figref#1{Figure~\ref{#1}}
\def\tabref#1{Table~\ref{#1}}
\def\eqref#1{Eq.~\ref{#1}}

\usepackage{titlesec}
\titlespacing*{\section}{0pt}{1.75ex plus .35ex minus .35ex}{0.875ex}
\titlespacing*{\subsection}{0pt}{1.4ex plus .35ex minus .35ex}{0.7ex}
\titlespacing*{\subsubsection}{0pt}{1.05ex plus .35ex minus .35ex}{0.525ex}
\titlespacing*{\paragraph}{0pt}{0.4375ex plus .175ex minus .175ex}{0.875em}

\usepackage{siunitx, array}
\newcolumntype{L}[1]{>{\raggedright\arraybackslash}p{#1}} 

\title{Unlocking Noise-Resistant Vision: Key Architectural Secrets for Robust Models}

\author{Bum Jun Kim, Makoto Kawano, Yusuke Iwasawa \& Yutaka Matsuo \\
Department of Engineering, The University of Tokyo \\
\texttt{\{bumjun.kim,kawano,iwasawa,matsuo\}@weblab.t.u-tokyo.ac.jp} \\
}

\iclrfinalcopy 
\begin{document}

\maketitle

\begin{abstract}
	While the robustness of vision models is often measured, their dependence on specific architectural design choices is rarely dissected. We investigate why certain vision architectures are inherently more robust to additive Gaussian noise and convert these empirical insights into simple, actionable design rules. Specifically, we performed extensive evaluations on 1,174 pretrained vision models, empirically identifying four consistent design patterns for improved robustness against Gaussian noise: larger stem kernels, smaller input resolutions, average pooling, and supervised vision transformers (ViTs) rather than CLIP ViTs, which yield up to 506 rank improvements and 21.6\%p accuracy gains. We then develop a theoretical analysis that explains these findings, converting observed correlations into causal mechanisms. First, we prove that low-pass stem kernels attenuate noise with a gain that decreases quadratically with kernel size and that anti-aliased downsampling reduces noise energy roughly in proportion to the square of the downsampling factor. Second, we demonstrate that average pooling is unbiased and suppresses noise in proportion to the pooling window area, whereas max pooling incurs a positive bias that grows slowly with window size and yields a relatively higher mean-squared error and greater worst-case sensitivity. Third, we reveal and explain the vulnerability of CLIP ViTs via a pixel-space Lipschitz bound: The smaller normalization standard deviations used in CLIP preprocessing amplify worst-case sensitivity by up to 1.91 times relative to the Inception-style preprocessing common in supervised ViTs. Our results collectively disentangle robustness into interpretable modules, provide a theory that explains the observed trends, and build practical, plug-and-play guidelines for designing vision models more robust against Gaussian noise.
\end{abstract}

\section{Introduction}
Vision models, implemented with deep neural networks, are now deployed across numerous fields, even in safety-critical applications ranging from medical imaging to autonomous driving. Their remarkable accuracy, however, conceals an uncomfortable fact: Performance can deteriorate when test images deviate---even slightly---from the training distribution \citep{DBLP:conf/iclr/HendrycksD19}. Even light Gaussian noise can trigger misclassifications, and in autonomous vehicles, such brittleness can lead to life-threatening failures.

Recent studies have empirically discovered that the architectural design of deep neural networks strongly shapes their robustness to common image transformations. Specifically, \citet{DBLP:conf/aaai/PaulC22,DBLP:conf/nips/BaiMYX21,DBLP:conf/nips/NaseerRKHKY21} observed that vision transformers (ViTs) often degrade less than previous convolutional networks, such as residual networks (ResNets), under various corruptions. Although promising results with ViTs have been reported, such studies typically treat each architecture as a whole, leaving unanswered which specific internal choices contribute to gains in robustness.

In this study, we dissect the robustness of vision models under Gaussian noise, showing that specific micro-architectural choices are key factors in determining robustness. We performed extensive experiments on available vision models from the \texttt{timm} library \citep{rw2019timm}, as well as controlled experiments; our empirical meta-analysis compares architectures pairwise within the vision models, which enables us to isolate the effect of each micro-architectural factor, thereby revealing four interesting design patterns in architectures that improve robustness against Gaussian noise:
\begin{itemize}[leftmargin=*,noitemsep,topsep=0pt,partopsep=0pt]
	\item \textbf{Larger} stem kernels, such as larger patch sizes in ViTs, rather than \textbf{smaller} ones,
	\item \textbf{Smaller} input resolutions, such as $224^2$, rather than \textbf{larger} ones, such as $384^2$,
	\item \textbf{Average} pooling, rather than \textbf{max} pooling, and
	\item \textbf{Supervised} learning ViTs, rather than \textbf{CLIP} ViTs.
\end{itemize}

Extending these empirical observations, we also derive several theoretical results that account for the differences in these choices. Specifically, we prove that noise gain decays quadratically with the stem kernel size and that downsampling after anti-alias filtering yields analogous gains (\secref{sec:ksrs}). Then we analyze Gaussian-noise error formulas for both pooling operators, showing that average pooling is unbiased with decreased variance, whereas max pooling incurs a positive bias and a higher mean-squared error (\secref{sec:pool}). Finally, we demonstrate that the vulnerability of CLIP ViTs is primarily caused by the choice of mean-std normalization, whose effect is proven with Lipschitz bounds (\secref{sec:clip}).

\section{Related Work}
\paragraph{Robustness literature and positioning of this study.} Robustness to common corruptions is typically evaluated using ImageNet-C \citep{DBLP:conf/iclr/HendrycksD19}. A consistent observation across studies is that ViTs often degrade less than CNNs do under such corruptions \citep{DBLP:conf/aaai/PaulC22,DBLP:conf/nips/BaiMYX21,DBLP:conf/nips/NaseerRKHKY21}. However, most prior comparisons treat architectures as monolithic families or vary training recipes, making it hard to isolate which micro-architectural choices drive robustness. Furthermore, multiple corruptions, such as brightness changes and blur, are mixed in. In contrast to these complex corruptions and architectures, we design a systematic evaluation protocol to isolate the effect of each micro-architectural factor. Furthermore, we select Gaussian noise due to its approximation of aggregate perturbations by the central limit theorem and its prevalence in real-world imaging, such as sensor readout and thermal noise. To this end, our experiments disentangle four architectural choices across pretrained models and controlled settings, enabling clean attribution. Our findings align with prior results on robustness studies \citep{DBLP:conf/aaai/PaulC22,DBLP:conf/icml/BoureauPL10} and add causal, quantitative explanations. The parts below review related work that corresponds to the empirical design patterns we identified for enhancing Gaussian noise robustness.

\paragraph{Anti-aliasing, kernels, and resolution.} Anti-aliased downsampling is known to reduce high-frequency sensitivity and improve stability \citep{DBLP:conf/icml/Zhang19,DBLP:journals/ijcv/ZouXYLL23}, and analogous ideas have been explored for ViTs \citep{DBLP:conf/nips/QianSZLJ21}. Complementing these studies, we provide explicit scaling laws: The output noise energy decays quadratically with the stem kernel size and the anti-aliased downsampling factor, explaining why larger stem kernels and smaller input resolutions improve robustness.

\paragraph{Pooling under additive noise.} Classical analysis shows that average pooling is unbiased with variance reduction, whereas max pooling introduces a positive bias under Gaussian noise \citep{DBLP:conf/icml/BoureauPL10}; recent studies further clarify when max pooling aids invariance despite worse noise behavior \citep{DBLP:journals/tmlr/MatobaDF23}. We extend this line and empirically verify the predicted advantage of average pooling over max pooling across multiple datasets.

\paragraph{Normalization, CLIP preprocessing, and Lipschitz sensitivity.} Vision models employ specific per-channel mean-std preprocessing, which, according to Lipschitz-based robustness theory \citep{DBLP:conf/nips/VirmauxS18,DBLP:journals/ml/GoukFPC21,DBLP:conf/nips/TsuzukuSS18}, directly rescales pixel-space sensitivity. We make this connection explicit: Smaller channel standard deviations enlarge the end-to-end Lipschitz bound, predicting greater worst-case and mean-squared sensitivity to additive noise.

\section{Noise Attenuation by Low-Pass Kernels}
\label{sec:ksrs}
ViTs have various configurations \citep{DBLP:conf/iclr/DosovitskiyB0WZ21}, such as the size of each patch in the patch embedding and the input image size in pixels, which we refer to as the input resolution. Even within the same ViT architecture, various pretrained weights are available: They were trained with different recipes, the hyperparameter combinations used in training. For example, \texttt{vit\_base\_patch16\_224.augreg\_in1k} indicates the ViT with a model size of base, a patch size of 16, a resolution of $224^2$, and pretrained weights obtained using a training recipe of AugReg \citep{DBLP:journals/tmlr/SteinerKZWUB22} and a dataset of ImageNet-1K \citep{DBLP:conf/cvpr/DengDSLL009}. Although plenty of variations in its configuration are allowed, the effect of each choice on robustness against Gaussian noise has not been clearly studied, making it difficult for practitioners to choose which one to use.

To study the effect of each architectural factor in a ViT on robustness, we performed an extensive evaluation using pretrained ViTs with various configurations. For example, by comparing \texttt{vit\_base\_patch16\_224.augreg\_in1k} and \texttt{vit\_base\_patch32\_224.augreg\_in1k}, we can study the effect of the choice of patch sizes of 16 and 32 on performance because all other conditions remained the same. In this section, we first present empirical observations from different configurations, and then we examine the corresponding properties.

\begin{table}[t!]
	\caption{Top-1 accuracy (\%) on the ImageNet-1K dataset before and after adding Gaussian noise to images. For the rank difference (RankDiff), more negative values indicate better robustness under noise. Models with large kernels and small resolutions consistently showed improved robustness.}
	\label{tab:timm}
	\begin{center}
		\resizebox{1.0\textwidth}{!}{
			\begin{tabular}{l|ccc}
				\toprule
				\textbf{Pretrained Model                                                         } & \textbf{Top-1 $\rightarrow$ w/ Noise                 } & \textbf{Rank $\rightarrow$ w/ Noise                 } & \textbf{RankDiff} \\
				\midrule
				\texttt{vit\_small\_patch16\_224.augreg\_in1k}                                     & 78.84                    $\rightarrow$ 59.22           & 885                      $\rightarrow$ 547            & -338              \\
				\texttt{vit\_small\_patch16\_384.augreg\_in1k}                                     & 81.12                    $\rightarrow$ 56.59           & 673                      $\rightarrow$ 613            & -60               \\
				\texttt{vit\_base\_patch16\_224.augreg\_in1k}                                      & 79.15                    $\rightarrow$ 62.21           & 862                      $\rightarrow$ 487            & -375              \\
				\texttt{vit\_base\_patch16\_384.augreg\_in1k}                                      & 81.10                    $\rightarrow$ 60.23           & 676                      $\rightarrow$ 524            & -152              \\
				\texttt{vit\_base\_patch32\_224.augreg\_in1k}                                      & 74.90                    $\rightarrow$ 58.44           & 1075                     $\rightarrow$ 569            & -506              \\
				\texttt{vit\_base\_patch32\_384.augreg\_in1k}                                      & 78.75                    $\rightarrow$ 59.65           & 893                      $\rightarrow$ 539            & -354              \\
				\midrule
				\texttt{vit\_tiny\_patch16\_224.augreg\_in21k\_ft\_in1k}                           & 75.46                    $\rightarrow$ 40.34           & 1060                     $\rightarrow$ 949            & -111              \\
				\texttt{vit\_tiny\_patch16\_384.augreg\_in21k\_ft\_in1k}                           & 78.42                    $\rightarrow$ 30.50           & 921                      $\rightarrow$ 1078           & +157              \\
				\texttt{vit\_small\_patch16\_224.augreg\_in21k\_ft\_in1k}                          & 81.39                    $\rightarrow$ 62.43           & 644                      $\rightarrow$ 479            & -165              \\
				\texttt{vit\_small\_patch16\_384.augreg\_in21k\_ft\_in1k}                          & 83.80                    $\rightarrow$ 62.25           & 349                      $\rightarrow$ 484            & +135              \\
				\texttt{vit\_small\_patch32\_224.augreg\_in21k\_ft\_in1k}                          & 76.00                    $\rightarrow$ 57.14           & 1044                     $\rightarrow$ 601            & -443              \\
				\texttt{vit\_small\_patch32\_384.augreg\_in21k\_ft\_in1k}                          & 80.48                    $\rightarrow$ 57.33           & 740                      $\rightarrow$ 596            & -144              \\
				\texttt{vit\_base\_patch8\_224.augreg\_in21k\_ft\_in1k}                            & 85.80                    $\rightarrow$ 73.50           & 145                      $\rightarrow$ 118            & -27               \\
				\texttt{vit\_base\_patch16\_224.augreg\_in21k\_ft\_in1k}                           & 84.53                    $\rightarrow$ 71.19           & 257                      $\rightarrow$ 192            & -65               \\
				\texttt{vit\_base\_patch16\_384.augreg\_in21k\_ft\_in1k}                           & 85.99                    $\rightarrow$ 70.89           & 129                      $\rightarrow$ 208            & +79               \\
				\texttt{vit\_base\_patch32\_224.augreg\_in21k\_ft\_in1k}                           & 80.71                    $\rightarrow$ 65.31           & 719                      $\rightarrow$ 392            & -327              \\
				\texttt{vit\_base\_patch32\_384.augreg\_in21k\_ft\_in1k}                           & 83.35                    $\rightarrow$ 63.72           & 412                      $\rightarrow$ 437            & +25               \\
				\texttt{vit\_large\_patch16\_224.augreg\_in21k\_ft\_in1k}                          & 85.84                    $\rightarrow$ 76.62           & 141                      $\rightarrow$ 55             & -86               \\
				\texttt{vit\_large\_patch16\_384.augreg\_in21k\_ft\_in1k}                          & 87.08                    $\rightarrow$ 76.23           & 59                       $\rightarrow$ 61             & +2                \\
				\midrule
				\texttt{vit\_base\_patch16\_224.orig\_in21k\_ft\_in1k}                             & 81.79                    $\rightarrow$ 60.91           & 603                      $\rightarrow$ 513            & -90               \\
				\texttt{vit\_base\_patch16\_384.orig\_in21k\_ft\_in1k}                             & 84.20                    $\rightarrow$ 54.91           & 302                      $\rightarrow$ 657            & +355              \\
				\midrule
				\texttt{vit\_base\_patch8\_224.augreg2\_in21k\_ft\_in1k}                           & 86.22                    $\rightarrow$ 76.09           & 109                      $\rightarrow$ 67             & -42               \\
				\texttt{vit\_base\_patch16\_224.augreg2\_in21k\_ft\_in1k}                          & 85.10                    $\rightarrow$ 74.50           & 203                      $\rightarrow$ 96             & -107              \\
				\midrule
				\texttt{vit\_base\_patch16\_224.sam\_in1k}                                         & 80.24                    $\rightarrow$ 57.13           & 771                      $\rightarrow$ 602            & -169              \\
				\texttt{vit\_base\_patch32\_224.sam\_in1k}                                         & 73.69                    $\rightarrow$ 51.33           & 1101                     $\rightarrow$ 748            & -353              \\
				\midrule
				\texttt{vit\_medium\_patch16\_gap\_256.sw\_in12k\_ft\_in1k}                        & 84.45                    $\rightarrow$ 73.07           & 274                      $\rightarrow$ 132            & -142              \\
				\texttt{vit\_medium\_patch16\_gap\_384.sw\_in12k\_ft\_in1k}                        & 85.54                    $\rightarrow$ 73.98           & 163                      $\rightarrow$ 106            & -57               \\
				\midrule
				\texttt{vit\_so150m\_patch16\_reg4\_gap\_256.sbb\_e250\_in12k\_ft\_in1k}           & 86.68                    $\rightarrow$ 77.54           & 81                       $\rightarrow$ 38             & -43               \\
				\texttt{vit\_so150m\_patch16\_reg4\_gap\_384.sbb\_e250\_in12k\_ft\_in1k}           & 87.37                    $\rightarrow$ 77.30           & 49                       $\rightarrow$ 44             & -5                \\
				\bottomrule
			\end{tabular}
		}
	\end{center}
\end{table}

\subsection{Empirical Observation} We used the \texttt{timm} library, which provides 1,174 pretrained vision models. For all pretrained models, we evaluated the top-1 accuracy (\%) on the standard ImageNet-1K dataset. Then we injected Gaussian noise into the images on the ImageNet-1K dataset and measured the top-1 accuracy. Although it is natural to observe a linear accuracy drop after applying a specific corruption \citep{DBLP:conf/icml/RechtRSS19,DBLP:conf/iclr/HendrycksD19}, a model with robustness would show a relatively smaller drop in top-1 accuracy. Motivated by this behavior, we identified robust models by observing relative ranking among the 1,174 models: When a model ranked 50th becomes 20th after adding Gaussian noise, we say that it demonstrates relatively stronger robustness to Gaussian noise. To investigate the model with improved rank, we computed the rank difference before and after applying Gaussian noise, where more negative values indicate better robustness. Full rationale for rank difference and technical details are available in \appref{app:rankdiff} and \appref{app:extsetup}. Full results on all models are in supplementary materials. Based on these rank differences, we compared pairs of ViTs with different configurations and investigated which architectural factors contribute to improved ranking under noise.

\tabref{tab:timm} summarizes the top-1 accuracy and ranking changes before and after injecting Gaussian noise. We found that the rank difference was lower when a ViT had 1) a larger patch size, such as 32, and 2) a smaller resolution, such as $224^2$. For example, comparing \texttt{vit\_base\_patch16\_224.augreg\_in1k} and \texttt{vit\_base\_patch32\_224.augreg\_in1k}, we observed that the model with a patch size of 32 yielded a lower rank difference than the one with a patch size of 16. We consistently observed similar behavior across multiple pretrained weights such as AugReg2, original ViTs, SAM, and others \citep{DBLP:journals/tmlr/SteinerKZWUB22,DBLP:conf/iclr/ChenHG22}. The same holds for resolution, where a model with a $224^2$ resolution exhibited a lower rank difference than one with $384^2$. Note that this observation is contrary to the common practice of scaling up resolution to improve general performance \citep{DBLP:conf/icml/TanL19}; our results indicate that this practice may increase vulnerability to Gaussian noise. These two factors were significantly more important than others, such as model size.

\begin{figure}[t!]
	\centering
	\begin{subfigure}{.5\textwidth}
		\centering
		\includegraphics[width=.99\linewidth]{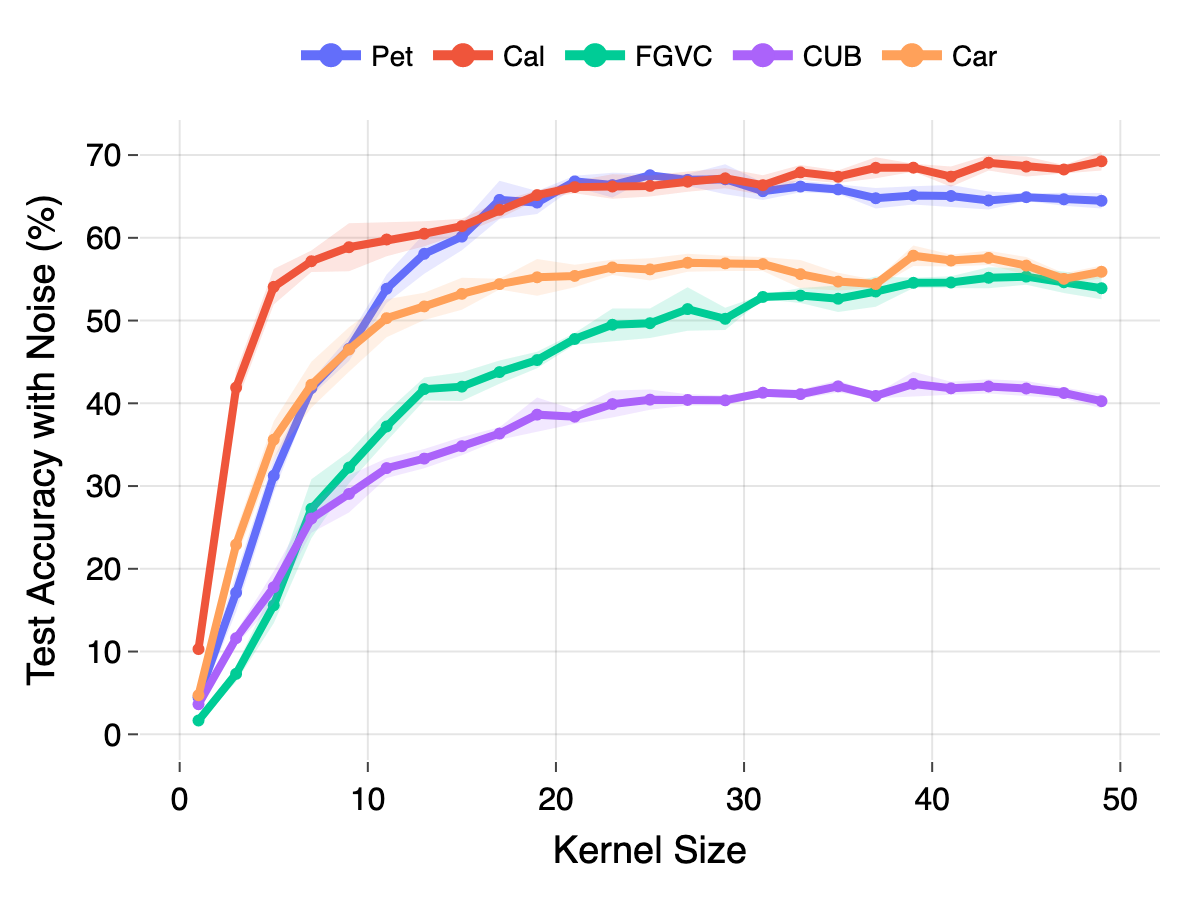}
	\end{subfigure}%
	\begin{subfigure}{.5\textwidth}
		\centering
		\includegraphics[width=.99\linewidth]{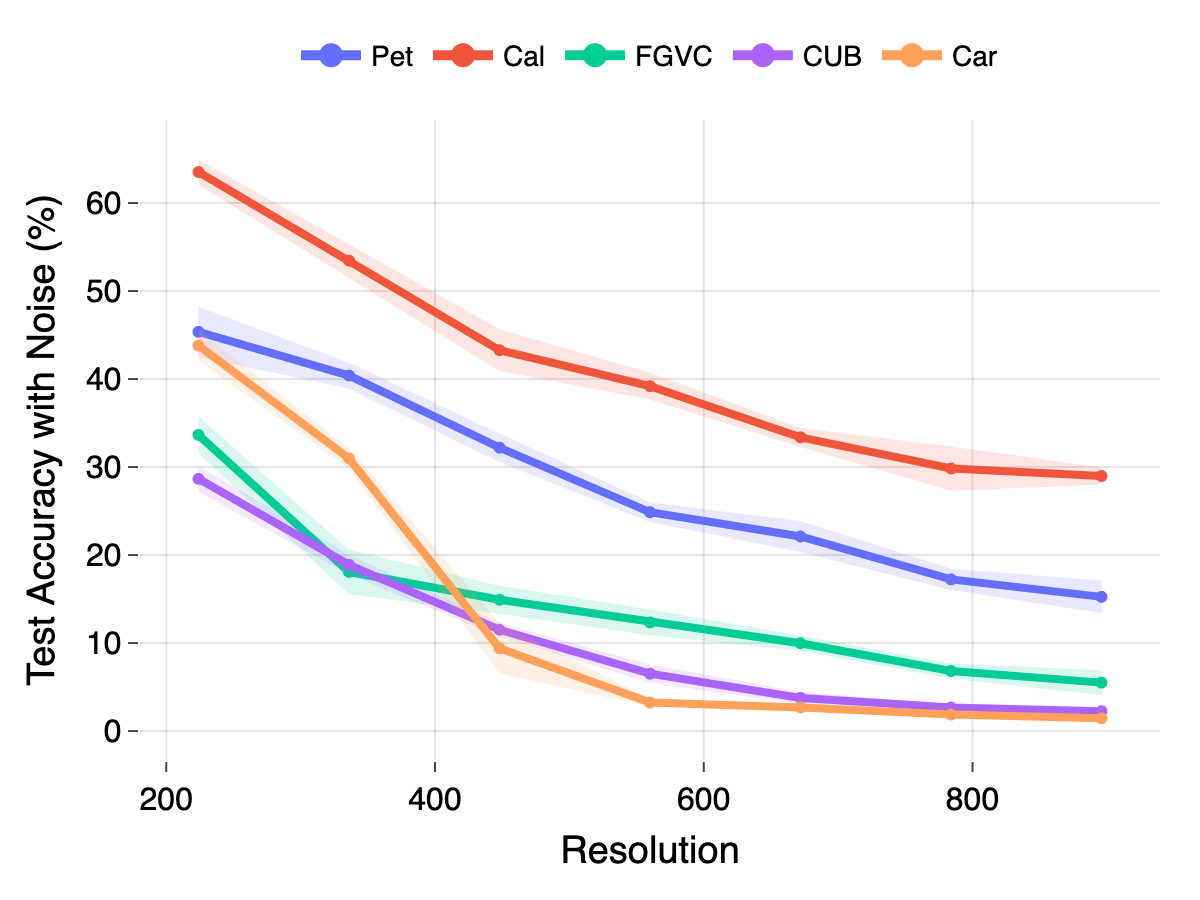}
	\end{subfigure}
	\vspace{-2.0em}
	\caption{Classification accuracy (\%) using ResNet-50 for different kernel sizes and resolutions. Larger kernels and smaller resolutions improved performance. Shaded areas represent standard deviations.}
	\label{fig:ksrs}
\end{figure}

The patch size of a ViT corresponds to the kernel size used in the patch embedding, which is referred to as the stem. Based on these observations, we investigated whether using a larger stem kernel and a smaller resolution improves robustness to Gaussian noise on another architecture, performing controlled experiments on ResNets \citep{DBLP:conf/cvpr/HeZRS16}. Specifically, we trained ResNet-50 on five datasets, including Oxford-IIIT Pet \citep{DBLP:conf/cvpr/ParkhiVZJ12}, Caltech-101 \citep{DBLP:journals/cviu/Fei-FeiFP07}, FGVC-Aircraft \citep{DBLP:journals/corr/MajiRKBV13}, Caltech-UCSD Birds-200-2011 \citep{wah2011caltech}, and Stanford Cars \citep{DBLP:conf/iccvw/Krause0DF13} datasets. Similar to the above ViT experiments, we trained ResNet in a standard recipe (\appref{app:extsetup}), obtained numerous models with different kernel sizes and resolutions, and measured classification accuracy after applying Gaussian noise.

We observed that larger kernel sizes and smaller resolutions improved classification accuracy under additive Gaussian noise (\figref{fig:ksrs}). The classification errors on noisy images tended to decrease quadratically with larger kernel sizes and smaller resolutions.

\subsection{Theoretical Analysis}
\label{sec:ksrs-theory}
Now, we prove that the noise energy decays quadratically with the stem kernel size and the resolution, or equivalently, the anti-aliased downsampling factor. Full proofs are available in \appref{app:proofs}. Throughout, $\eta\sim\mathcal N(0,\sigma^2 I)$ denotes independent and identically distributed (i.i.d.) Gaussian noise, and the per-pixel noise gain is the output noise energy normalized by the number of output pixels \citep{oppenheim1999discrete}.

\paragraph{Setup.} For a kernel size $k\ge3$, let $K_k\in\mathbb{R}^{k\times k}$ denote the linear, shift-invariant stem kernel, and $\widehat K_k$ its DFT \citep{oppenheim1999discrete}. We consider a single, mild assumption on the stem kernel:
\begin{itemize}[leftmargin=1.25em, itemsep=2pt]
	\item \textbf{(A\textsubscript{roll})} Radial low-pass envelope at scale $1/k$: There exist $\beta,\delta>0$ such that, for all frequencies $\boldsymbol\omega$,
	      \[
		      |\widehat K_k(\boldsymbol\omega)| \le \phi_k(\|\boldsymbol\omega\|),
		      \qquad \phi_k(r) \coloneqq (1+\beta k r)^{-1-\delta},
	      \]
	      where $\phi_k$ is nonincreasing in $r$. This assumption works well in practical use cases (\appref{app:simul}).
\end{itemize}

\paragraph{Per-pixel noise gain for stem kernel.} We define
\begin{equation}\label{eq:noise-gain-stem}
	\gamma(k) \coloneqq \frac{\mathbb{E} \bigl[\|K_k*\eta\|_2^2\bigr]}{\sigma^2 HW}
	\stackrel{\text{(Parseval)}}=\frac{1}{HW}\sum_{\boldsymbol\omega} |\widehat K_k(\boldsymbol\omega)|^2
	= \|K_k\|_{F}^{2},
\end{equation}
where $H$ and $W$ are the height and width. Intuitively, $\gamma(k)$ is the average squared magnitude response of the stem kernel.

\begin{theorem}[Noise attenuation for practical low-pass stem kernel]\label{thm:lowpassstems}
	Assume \textup{(A\textsubscript{roll})}. Then, there exists a constant $C>0$, independent of $k$, such that
	\[
		\gamma(k)=\frac{\mathbb{E} \bigl[\|K_k*\eta\|_2^2\bigr]}{\sigma^2 HW}
		\le \frac{C}{k^{2}}.
	\]
	Moreover, the $k^{-2}$ rate is achievable.
\end{theorem}

\begin{remark}[Practical reading of Theorem~\ref{thm:lowpassstems}]
	Doubling the stem kernel size, such as the patch size from 16 to 32, quarters the output noise energy ($\approx -6$ dB).
\end{remark}

\paragraph{Per-output-pixel noise gain for anti-aliased downsampling.} For a downsampling factor $s\ge1$, we define
\begin{equation}\label{eq:Ds-def}
	D_s \coloneqq (\Downarrow_s)\circ K_{g(s)},\qquad c_1 s \le g(s) \le c_2 s,
\end{equation}
\ie, filter with $K_{g(s)}$ satisfying the same assumptions at scale $g(s)$ and then downsample by $s$. We normalize the noise gain by the number of output pixels:
\begin{equation}\label{eq:noise-gain-down}
	\gamma_{\downarrow}(s) \coloneqq \frac{\mathbb{E} \bigl[\|D_s\eta\|_2^2\bigr]}{\sigma^2 HW/s^{2}}.
\end{equation}

\begin{theorem}[Resolution-driven robustness]\label{thm:resolution}
	There exists $C'>0$ independent of $s$ such that
	\[
		\gamma_{\downarrow}(s) \le \frac{C'}{s^{2}}.
	\]
	This $s^{-2}$ rate is tight up to constants.
\end{theorem}

\begin{remark}[Practical reading of Theorem~\ref{thm:resolution}]
	Resizing $384^2$ to $224^2$ corresponds to $s\approx 1.71$ and yields roughly $s^{-2} \approx 0.34$ of the original noise energy per output pixel ($\approx -4.7$ dB).
\end{remark}

\section{Choice on Pooling}
\label{sec:pool}

\subsection{Empirical Observation} Extending the above analysis, we probed the effect of choosing specific architectural types of ResNets on robustness. Specifically, ResNet has several types, including ResNet-\{C, D, T, S\} \citep{DBLP:conf/cvpr/HeZ0ZXL19,DBLP:journals/corr/abs-2110-00476,DBLP:journals/jmlr/GuoHHLLLSWXZZ0Z20}, although the effects of these choices and their underlying mechanisms on robustness have been rarely studied. Here, we trained the four ResNets on the five datasets mentioned above and compared their classification accuracy after applying Gaussian noise (\tabref{tab:type}).

\begin{table}[t!]
	\caption{Classification accuracy (\%) with different choices of ResNet type. The numbers in the parentheses represent standard deviations on the five runs with different random seeds.}
	\label{tab:type}
	\begin{center}
		\begin{tabular}{ll|cccc}
			\toprule
			\textbf{Dataset                          } & \textbf{Model              } & \textbf{ResNet-50-T } & \textbf{ResNet-50-D } & \textbf{ResNet-50-C } & \textbf{ResNet-50-S } \\
			\midrule
			\multirow{2}{*}{Oxford-IIIT Pet}           & Val. Acc. w/ Noise           & 39.1 (11.1)           & 37.9 (9.6)            & 34.9 (11.4)           & 24.3 (3.3)            \\
			                                           & Test Acc. w/ Noise           & 38.1 (11.3)           & 36.1 (10.3)           & 34.0 (10.4)           & 22.9 (2.4)            \\
			\midrule
			\multirow{2}{*}{Caltech-101}               & Val. Acc. w/ Noise           & 62.3 (1.4)            & 61.2 (3.0)            & 58.8 (1.1)            & 50.9 (3.2)            \\
			                                           & Test Acc. w/ Noise           & 59.7 (1.1)            & 59.1 (2.8)            & 57.8 (1.0)            & 49.5 (2.7)            \\
			\midrule
			\multirow{2}{*}{FGVC-Aircraft}             & Val. Acc. w/ Noise           & 27.8 (1.6)            & 27.3 (2.4)            & 23.9 (1.9)            & 4.7  (0.9)            \\
			                                           & Test Acc. w/ Noise           & 29.9 (1.1)            & 30.4 (1.6)            & 26.1 (2.1)            & 5.5  (0.8)            \\
			\midrule
			\multirow{2}{*}{\shortstack[l]{Caltech-UCSD                                                                                                                               \\ Birds-200-2011}} & Val. Acc. w/ Noise   & 27.6 (2.0)  & 28.8 (0.8)  & 26.3 (1.3)  & 13.9 (0.6)  \\
			                                           & Test Acc. w/ Noise           & 26.3 (2.0)            & 27.7 (0.6)            & 25.2 (1.7)            & 13.7 (1.1)            \\
			\midrule
			\multirow{2}{*}{Stanford Cars}             & Val. Acc. w/ Noise           & 56.9 (2.3)            & 55.2 (2.8)            & 41.6 (2.3)            & 29.2 (1.9)            \\
			                                           & Test Acc. w/ Noise           & 55.0 (1.9)            & 53.2 (2.7)            & 40.5 (2.3)            & 28.5 (2.0)            \\
			\bottomrule
		\end{tabular}
	\end{center}
\end{table}

Overall, the T and D types of ResNet demonstrated robust results against Gaussian noise, followed by the C and S types of ResNet. While there are several different factors among the four ResNets (\appref{app:extsetup}), the core difference is the pooling in downsampling: the T and D types of ResNet adopt average pooling with convolution in downsampling, whereas the C and S types of ResNet adopt strided $1 \times 1$ convolution in downsampling, which is equivalent to nearest-neighbor pooling followed by a $1 \times 1$ convolution.

\begin{table}[t!]
	\caption{Classification accuracy (\%) comparing different poolings. The largest gain came from AvgPool.}
	\label{tab:ap}
	\begin{center}
		\begin{tabular}{ll|ccc}
			\toprule
			\textbf{Dataset                          } & \textbf{Model              } & \textbf{MaxPool                             } & \textbf{NNPool     } & \textbf{AvgPool    } \\
			\midrule
			\multirow{2}{*}{Oxford-IIIT Pet}           & Val. Acc. w/ Noise           & 42.0                          (1.1)           & 44.2 (2.8)           & 50.2 (1.9)           \\
			                                           & Test Acc. w/ Noise           & 41.8                          (0.9)           & 42.3 (3.2)           & 49.3 (1.8)           \\
			\midrule
			\multirow{2}{*}{Caltech-101}               & Val. Acc. w/ Noise           & 59.5                          (1.0)           & 58.3 (1.1)           & 62.7 (1.8)           \\
			                                           & Test Acc. w/ Noise           & 57.2                          (1.3)           & 56.7 (1.1)           & 60.8 (1.9)           \\
			\midrule
			\multirow{2}{*}{FGVC-Aircraft}             & Val. Acc. w/ Noise           & 24.2                          (3.5)           & 22.8 (1.9)           & 41.0 (2.9)           \\
			                                           & Test Acc. w/ Noise           & 27.3                          (3.6)           & 24.7 (1.8)           & 43.1 (3.1)           \\
			\midrule
			\multirow{2}{*}{\shortstack[l]{Caltech-UCSD                                                                                                                             \\ Birds-200-2011}} & Val. Acc. w/ Noise   & 26.9                          (1.8) & 27.5 (3.0) & 28.8 (1.8) \\
			                                           & Test Acc. w/ Noise           & 26.1                          (1.7)           & 25.6 (2.6)           & 26.8 (1.2)           \\
			\midrule
			\multirow{2}{*}{Stanford Cars}             & Val. Acc. w/ Noise           & 43.3                          (3.4)           & 49.1 (1.9)           & 52.1 (1.8)           \\
			                                           & Test Acc. w/ Noise           & 42.2                          (2.8)           & 46.9 (1.3)           & 51.2 (2.1)           \\
			\bottomrule
		\end{tabular}
	\end{center}
\end{table}

We further explored the effect of pooling choice on robustness to Gaussian noise. Using ResNet-50, we compared the original one, which uses max pooling in the stem, and modified ResNets that adopt nearest-neighbor pooling or average pooling in the stem (\tabref{tab:ap}). ResNets with average pooling consistently yielded robust performance against Gaussian noise among the three setups in pooling. More results for other architectures under controlled conditions are available in \appref{app:additional}.

\subsection{Theoretical Analysis}
\label{sec:pool-theory}
We explain why average pooling is more robust than max pooling under i.i.d. additive Gaussian noise.

\paragraph{Setup.} Consider a pooling window of size $k\ge 2$ in a single channel. Let the clean activations be $S=(S_1,\dots,S_k)\in\mathbb{R}^k$ and the observation be $S+\eta$ with i.i.d. noise $\eta\sim\mathcal N(0,\sigma^2 I_k)$. We define
\[
	X_{\mathrm{avg}} \coloneqq \tfrac1k\sum_{i=1}^k(S_i+\eta_i),\qquad
	X_{\mathrm{max}} \coloneqq \max_{1\le i\le k}(S_i+\eta_i),
\]
their clean counterparts $S_{\mathrm{avg}} \coloneqq \tfrac1k\sum_i S_i$, $S_{\mathrm{max}} \coloneqq \max_i S_i$, and the errors $\delta_{\mathrm{avg}} \coloneqq X_{\mathrm{avg}}-S_{\mathrm{avg}}$, $\delta_{\mathrm{max}} \coloneqq X_{\mathrm{max}}-S_{\mathrm{max}}$. Let $\Delta \coloneqq S_{(1)}-S_{(2)}\ge 0$ be the gap between the largest and second-largest entries. We also denote $T_{\mathrm{avg}}(v)\coloneqq\tfrac{1}{k}\sum_{i=1}^{k} v_i$, $T_{\max}(v) \coloneqq \max_{1\le i\le k} v_i$, and $\|T\|_{\ell_2\to\ell_2}$ for $\ell_2$-Lipschitz constant.

\begin{theorem}[Average and max poolings under Gaussian noise]
	\label{thm:pooling}
	For any $S\in\mathbb{R}^k$ and $\sigma>0$, we have
	\begin{enumerate}[label=(\roman*), leftmargin=1.4em, itemsep=3pt]
		\item \textbf{Average pooling} is unbiased and reduces variance proportionally to the window area:
		      $\mathbb E[\delta_{\mathrm{avg}}]=0, \Var[\delta_{\mathrm{avg}}]=\sigma^2/k.$
		\item \textbf{Max pooling} incurs a positive noise bias and admits the following mean-squared error (MSE) controls:
		      \begin{align*}
			       & \text{(Bias)}\quad
			      \mathbb E[\delta_{\mathrm{max}}]
			      = \mathbb E[\max_i(S_i+\eta_i)]-\max_i S_i \ge 0,             \\
			       & \text{(Uniform-signal case)}\quad
			      (S_1=\cdots=S_k):\
			      \delta_{\mathrm{max}}=\sigma M_k,\
			      \mathbb E[\delta_{\mathrm{max}}^2]=\sigma^2 \mathbb E[M_k^2], \\
			       & \text{(General case)}\quad
			      |\delta_{\mathrm{max}}|\le \|\eta\|_\infty \Rightarrow\
			      \mathbb E[\delta_{\mathrm{max}}^2]\le \sigma^2 \mathbb E[A_k^2],
		      \end{align*}
		      where $M_k \coloneqq \max_{1\le i\le k} Z_i$ and $A_k \coloneqq \max_{1\le i\le k}|Z_i|$ with $Z_i\stackrel{\mathrm{i.i.d.}}{\sim}\mathcal N(0,1)$. In particular, $\mathbb E[A_k^2]\le 2\log(2k)+2$, so $\mathbb E[\delta_{\mathrm{max}}^2]\le \sigma^2\big(2\log(2k)+2\big)$.
		\item \textbf{Adversarial worst-case sensitivity.} For any perturbation $n\in\mathbb R^k$, $|\tfrac1k\sum_i n_i |\le \|n\|_2/\sqrt{k}$, so $\|T_{\mathrm{avg}}\|_{\ell_2\to\ell_2}=k^{-1/2}$; and $|\max_i a_i-\max_i b_i|\le \|a-b\|_\infty\le \|a-b\|_2$, so $\|T_{\mathrm{max}}\|_{\ell_2\to\ell_2}\le 1$.
		\item \textbf{Large-gap regime.} For $z\coloneqq \Delta/\sigma$, one has $\lim_{z\to\infty}\mathbb E[\delta_{\mathrm{max}}^2]=\sigma^2$; when the top index never switches under noise, max pooling is equivalent to reading a single noisy entry.
	\end{enumerate}
\end{theorem}

\begin{remark}[Practical reading of Theorem~\ref{thm:pooling}]
	Average pooling is unbiased and cuts Gaussian noise variance by a factor $k$ (\eg, a $2\times2$ window gives $-6$ dB). Max pooling is positively biased, and its MSE grows at most logarithmically with the window size, while also having a greater worst-case Lipschitz gain, clearly worse than average pooling.
\end{remark}

\begin{remark}[Average and nearest-neighbor poolings]
	Selecting a fixed element in the window, such as the nearest-neighbor pooling, is unbiased with an MSE $\sigma^2$. Hence, average pooling is strictly more robust to additive Gaussian noise than nearest-neighbor pooling by a factor of $k$ in MSE.
\end{remark}

\section{Why are CLIP models vulnerable?}
\label{sec:clip}

\subsection{Empirical Observation}
Although the original ViT \citep{DBLP:conf/iclr/DosovitskiyB0WZ21} was trained with supervised learning, the CLIP study \citep{DBLP:conf/icml/RadfordKHRGASAM21} trained ViTs with self-supervised learning and successfully achieved competitive performance. Currently, available pretrained weights for ViTs are largely divided into CLIP ViTs and others trained with supervised learning; we refer to the latter as supervised ViTs. The training methods and datasets differ between these two sources of ViTs, yielding different pretrained weights, while they have almost the same architecture with only a single minor difference. Nevertheless, we observed that CLIP ViTs exhibited significant vulnerabilities when Gaussian noise was applied to images (\tabref{tab:clip}). Similar observations regarding the degraded performance of CLIP ViTs due to distribution shifts have been reported in certain studies \citep{DBLP:conf/icml/ShuG0W0L23,DBLP:conf/cvpr/WortsmanIKLKRLH22}; they focused on the characteristics of CLIP pretrained weights due to different datasets or training schemes, but we present a different perspective on this issue.

We performed ablation studies to identify what determined the difference in robustness (\appref{app:ablation}). We discovered that the core factor in different robustness arose from the preprocessing pipeline. Specifically, CLIP ViTs apply mean-std normalization to input images using certain per-channel mean and standard deviation (std) constants, which we refer to as the \texttt{OPENAI} constants (\appref{app:extsetup}), whereas supervised ViTs apply different per-channel mean-std constants, which are often called \texttt{INCEPTION} constants \citep{DBLP:conf/cvpr/SzegedyVISW16}. In other words, the \texttt{OPENAI} mean-std constants led to vulnerability to Gaussian noise, whereas the \texttt{INCEPTION} mean-std constants did not show this vulnerability.

Indeed, when we replaced the \texttt{OPENAI} mean-std constants with the \texttt{INCEPTION} constants, the CLIP ViTs achieved improved robustness (\tabref{tab:clippet}). The reverse also holds, and similar vulnerability was observed when adopting \texttt{IMAGENET} mean-std constants for ViTs. Full results on other datasets are available in \appref{app:additional}, where we observed these improvements across various pretrained weights with different training recipes.

\begin{table}[t!]
	\caption{ImageNet-1K results for ViT-B/16 $224^2$ with eight different pretrained weights. CLIP ViTs tended to yield worse ranks under noise.}
	\label{tab:clip}
	\begin{center}
		\resizebox{1.0\textwidth}{!}{
			\begin{tabular}{ll|ccc}
				\toprule
				\textbf{Pretrained Model                                                } & \textbf{Mean-Std           } & \textbf{Top-1 $\rightarrow$ w/ Noise                  } & \textbf{Rank $\rightarrow$ w/ Noise                } & \textbf{RankDiff} \\
				\midrule
				\texttt{vit\_base\_patch16\_224.augreg\_in1k}                             & \texttt{INCEPTION}           & 79.15                     $\rightarrow$ 62.21           & 862                      $\rightarrow$ 487           & -375              \\
				\texttt{vit\_base\_patch16\_224.augreg2\_in21k\_ft\_in1k}                 & \texttt{INCEPTION}           & 85.10                     $\rightarrow$ 74.50           & 203                      $\rightarrow$ 96            & -107              \\
				\texttt{vit\_base\_patch16\_224.orig\_in21k\_ft\_in1k}                    & \texttt{INCEPTION}           & 81.79                     $\rightarrow$ 60.91           & 603                      $\rightarrow$ 513           & -90               \\
				\texttt{vit\_base\_patch16\_224.augreg\_in21k\_ft\_in1k}                  & \texttt{INCEPTION}           & 84.53                     $\rightarrow$ 71.19           & 257                      $\rightarrow$ 192           & -65               \\
				\texttt{vit\_base\_patch16\_clip\_224.openai\_ft\_in12k\_in1k}            & \texttt{OPENAI}              & 85.94                     $\rightarrow$ 70.81           & 135                      $\rightarrow$ 209           & +74               \\
				\texttt{vit\_base\_patch16\_clip\_224.laion2b\_ft\_in12k\_in1k}           & \texttt{OPENAI}              & 86.17                     $\rightarrow$ 71.24           & 114                      $\rightarrow$ 189           & +75               \\
				\texttt{vit\_base\_patch16\_clip\_224.laion2b\_ft\_in1k}                  & \texttt{OPENAI}              & 85.47                     $\rightarrow$ 67.88           & 168                      $\rightarrow$ 311           & +143              \\
				\texttt{vit\_base\_patch16\_clip\_224.openai\_ft\_in1k}                   & \texttt{OPENAI}              & 85.29                     $\rightarrow$ 67.06           & 182                      $\rightarrow$ 340           & +158              \\
				\bottomrule
			\end{tabular}
		}
	\end{center}
\end{table}

\begin{table}[t!]
	\caption{Classification accuracy (\%) for fine-tuning ViTs on the Oxford-IIIT Pet.}
	\label{tab:clippet}
	\begin{center}
		\resizebox{1.0\textwidth}{!}{
			\begin{tabular}{ll|cc}
				\toprule
				\textbf{Pretrained Model                                               } & \textbf{Mean-Std           } & \textbf{Val. Acc. w/ Noise                          } & \textbf{Test Acc. w/ Noise                                  } \\
				\midrule
				\texttt{vit\_base\_patch16\_clip\_224.openai\_ft\_in12k\_in1k}           & \texttt{OPENAI}              & 94.5 (1.0)       $\rightarrow$   77.7 (3.4)           & 93.8 (1.0)                 $\rightarrow$ 76.3 (4.1)           \\
				\texttt{vit\_base\_patch16\_clip\_224.openai\_ft\_in12k\_in1k}           & \texttt{INCEPTION}           & 95.5 (0.5)       $\rightarrow$   87.3 (2.1)           & 95.2 (0.6)                 $\rightarrow$ 87.2 (2.2)           \\
				\texttt{vit\_base\_patch16\_clip\_224.openai\_ft\_in12k\_in1k}           & \texttt{IMAGENET}            & 94.2 (0.4)       $\rightarrow$   73.9 (2.3)           & 93.4 (0.5)                 $\rightarrow$ 72.7 (2.5)           \\
				\texttt{vit\_base\_patch16\_clip\_224.datacompxl}                        & \texttt{OPENAI}              & 93.6 (0.9)       $\rightarrow$   67.4 (6.0)           & 93.2 (0.9)                 $\rightarrow$ 67.3 (5.9)           \\
				\texttt{vit\_base\_patch16\_clip\_224.datacompxl}                        & \texttt{INCEPTION}           & 94.7 (0.5)       $\rightarrow$   78.5 (4.0)           & 93.6 (0.6)                 $\rightarrow$ 78.4 (3.8)           \\
				\texttt{vit\_base\_patch16\_clip\_224.datacompxl}                        & \texttt{IMAGENET}            & 92.8 (0.9)       $\rightarrow$   57.6 (7.4)           & 92.6 (0.5)                 $\rightarrow$ 58.1 (7.4)           \\
				\texttt{vit\_base\_patch16\_clip\_224.dfn2b}                             & \texttt{OPENAI}              & 95.0 (0.3)       $\rightarrow$   73.1 (1.5)           & 94.1 (0.5)                 $\rightarrow$ 73.3 (1.9)           \\
				\texttt{vit\_base\_patch16\_clip\_224.dfn2b}                             & \texttt{INCEPTION}           & 94.8 (0.8)       $\rightarrow$   78.6 (4.9)           & 93.6 (0.4)                 $\rightarrow$ 79.8 (5.0)           \\
				\texttt{vit\_base\_patch16\_clip\_224.dfn2b}                             & \texttt{IMAGENET}            & 95.1 (0.3)       $\rightarrow$   69.8 (2.7)           & 94.0 (0.4)                 $\rightarrow$ 68.8 (3.1)           \\
				\texttt{vit\_base\_patch16\_clip\_224.metaclip\_2pt5b}                   & \texttt{OPENAI}              & 92.8 (0.7)       $\rightarrow$   64.8 (4.4)           & 92.0 (0.7)                 $\rightarrow$ 62.3 (3.9)           \\
				\texttt{vit\_base\_patch16\_clip\_224.metaclip\_2pt5b}                   & \texttt{INCEPTION}           & 94.7 (0.4)       $\rightarrow$   78.5 (2.0)           & 93.9 (0.3)                 $\rightarrow$ 78.5 (1.8)           \\
				\texttt{vit\_base\_patch16\_clip\_224.metaclip\_2pt5b}                   & \texttt{IMAGENET}            & 91.6 (0.3)       $\rightarrow$   54.5 (2.5)           & 90.8 (0.3)                 $\rightarrow$ 52.8 (1.6)           \\
				\texttt{vit\_base\_patch16\_clip\_224.openai}                            & \texttt{OPENAI}              & 92.5 (0.3)       $\rightarrow$   71.7 (1.0)           & 91.9 (0.6)                 $\rightarrow$ 70.2 (1.2)           \\
				\texttt{vit\_base\_patch16\_clip\_224.openai}                            & \texttt{INCEPTION}           & 94.0 (0.7)       $\rightarrow$   78.6 (4.6)           & 93.2 (0.9)                 $\rightarrow$ 77.3 (5.1)           \\
				\texttt{vit\_base\_patch16\_clip\_224.openai}                            & \texttt{IMAGENET}            & 91.2 (0.5)       $\rightarrow$   58.5 (4.0)           & 90.7 (0.8)                 $\rightarrow$ 58.4 (4.3)           \\
				\texttt{vit\_base\_patch16\_clip\_224.laion2b}                           & \texttt{OPENAI}              & 91.8 (1.2)       $\rightarrow$   56.1 (7.7)           & 90.5 (1.1)                 $\rightarrow$ 54.0 (6.6)           \\
				\texttt{vit\_base\_patch16\_clip\_224.laion2b}                           & \texttt{INCEPTION}           & 93.8 (0.6)       $\rightarrow$   76.4 (1.9)           & 92.8 (0.5)                 $\rightarrow$ 75.6 (1.8)           \\
				\texttt{vit\_base\_patch16\_clip\_224.laion2b}                           & \texttt{IMAGENET}            & 90.2 (0.8)       $\rightarrow$   52.3 (4.4)           & 89.5 (0.8)                 $\rightarrow$ 51.4 (4.1)           \\
				\texttt{vit\_base\_patch16\_224.augreg\_in1k}                            & \texttt{OPENAI}              & 95.5 (0.2)       $\rightarrow$   88.7 (0.3)           & 94.9 (0.2)                 $\rightarrow$ 88.2 (0.7)           \\
				\texttt{vit\_base\_patch16\_224.augreg\_in1k}                            & \texttt{INCEPTION}           & 95.5 (0.1)       $\rightarrow$   89.7 (0.5)           & 94.4 (0.3)                 $\rightarrow$ 89.2 (0.8)           \\
				\texttt{vit\_base\_patch16\_224.augreg\_in1k}                            & \texttt{IMAGENET}            & 95.5 (0.2)       $\rightarrow$   87.7 (0.5)           & 94.9 (0.2)                 $\rightarrow$ 87.9 (0.7)           \\
				\texttt{vit\_base\_patch16\_224.augreg\_in21k}                           & \texttt{OPENAI}              & 95.6 (0.3)       $\rightarrow$   91.4 (0.3)           & 95.2 (0.5)                 $\rightarrow$ 91.9 (0.6)           \\
				\texttt{vit\_base\_patch16\_224.augreg\_in21k}                           & \texttt{INCEPTION}           & 95.9 (0.2)       $\rightarrow$   92.3 (0.3)           & 95.6 (0.4)                 $\rightarrow$ 92.6 (0.4)           \\
				\texttt{vit\_base\_patch16\_224.augreg\_in21k}                           & \texttt{IMAGENET}            & 95.7 (0.5)       $\rightarrow$   91.6 (0.5)           & 95.6 (0.3)                 $\rightarrow$ 92.0 (0.5)           \\
				\texttt{vit\_base\_patch16\_224.mae}                                     & \texttt{OPENAI}              & 93.5 (0.3)       $\rightarrow$   70.8 (2.8)           & 93.4 (0.2)                 $\rightarrow$ 72.7 (2.3)           \\
				\texttt{vit\_base\_patch16\_224.mae}                                     & \texttt{INCEPTION}           & 93.7 (0.3)       $\rightarrow$   75.0 (2.1)           & 93.3 (0.2)                 $\rightarrow$ 75.2 (2.5)           \\
				\texttt{vit\_base\_patch16\_224.mae}                                     & \texttt{IMAGENET}            & 93.5 (0.3)       $\rightarrow$   72.0 (2.0)           & 92.7 (0.5)                 $\rightarrow$ 71.9 (2.2)           \\
				\bottomrule
			\end{tabular}
		}
	\end{center}
\end{table}

\subsection{Theoretical Analysis}
We give an explanation for the empirical vulnerability of CLIP ViTs to additive Gaussian noise. The key point is that channel-wise normalization sets the pixel-space sensitivity scale: Smaller per-channel stds in the input normalization enlarge the worst-case response to perturbations even before the backbone acts.

\paragraph{Setup.} Let $x\in[0,1]^{C\times H\times W}$ be an image and $\eta$ an additive perturbation. Let $\mu\in\mathbb R^C$ and $\boldsymbol\sigma\in\mathbb R_{>0}^C$ be the per-channel means and stds, and define the normalization $N_{\mu,\boldsymbol\sigma}(x) \coloneqq (x-\mu)/\boldsymbol\sigma$. Let $f:\mathbb R^{C\times H\times W} \to \mathbb R^K$ denote the vision backbone operating on normalized inputs, which is globally $\ell_2$-Lipschitz with constant $L_z$ on its domain.\footnote{This assumption holds when linear layers have bounded spectral norms and other modules are Lipschitz. ReLU: $1$-Lipschitz~\citep{DBLP:journals/ml/GoukFPC21}; GELU: $\approx 1.13$~\citep{DBLP:journals/corr/HendrycksG16}; LayerNorm: Lipschitz with a constant set by $\gamma$ and $\varepsilon$~\citep{DBLP:journals/corr/BaKH16}.}  We study the end-to-end pipeline $F_{\mu,\boldsymbol\sigma} \coloneqq f\circ N_{\mu,\boldsymbol\sigma}$ and its $\ell_2$-Lipschitz constant $\|F_{\mu,\boldsymbol\sigma}\|_{\mathrm{Lip}}$.

\begin{theorem}[Pixel-space Lipschitz bound]
	\label{thm:clip}
	For any image $x$ and perturbation $\eta$, we obtain
	\[
		\big\|F_{\mu,\boldsymbol\sigma}(x+\eta)-F_{\mu,\boldsymbol\sigma}(x)\big\|_2
		\le L_z\Big\|\frac{\eta}{\boldsymbol\sigma}\Big\|_2
		\le \frac{L_z}{\sigma_{\min}}\|\eta\|_2,
	\]
	where $\sigma_{\min} \coloneqq \min_c \boldsymbol\sigma_c$. In particular, the pixel-space Lipschitz constant satisfies $\|F_{\mu,\boldsymbol\sigma}\|_{\mathrm{Lip}}\le L_z/\sigma_{\min}$.
\end{theorem}

\begin{proof}
	Write $z=N_{\mu,\boldsymbol\sigma}(x)$ and $\tilde z=N_{\mu,\boldsymbol\sigma}(x+\eta)=z+\eta/\boldsymbol\sigma$. By Lipschitzness of $f$, we have $\|f(\tilde z)-f(z)\|_2\le L_z\|\eta/\boldsymbol\sigma\|_2\le (L_z/\sigma_{\min})\|\eta\|_2$.
\end{proof}

\begin{remark}[Practical reading of Theorem~\ref{thm:clip}]
	For the standard choices
	\[
		\boldsymbol\sigma_{\text{INCEPTION}}=(0.5, 0.5, 0.5),\qquad
		\boldsymbol\sigma_{\text{CLIP}} = (0.26862954, 0.26130258, 0.27577711),
	\]
	the worst-case pixel-space sensitivity bound for CLIP is greater by a factor
	\[
		\frac{L_z/\min(\boldsymbol\sigma_{\text{CLIP}})}{L_z/\min(\boldsymbol\sigma_{\text{INCEPTION}})}
		=\frac{0.5}{0.26130258} \approx 1.91,
	\]
	relative to a supervised ViT using \texttt{INCEPTION} statistics. This $\sim 1.91\times$ looser bound amplifies the effect of input perturbations before the feature extractor.
\end{remark}

\section{Conclusion}
Across \texttt{timm} models and controlled experiments, four design patterns consistently improved robustness: (1) larger stem kernel sizes, (2) smaller resolutions, (3) average pooling instead of max pooling, and (4) supervised ViTs rather than CLIP ViTs. Practically, we recommend models with these design patterns such as \texttt{vit\_base\_patch32\_224.augreg\_in21k\_ft\_in1k} for ViT-B as an example. Our analysis integrates these findings: Theorem~\ref{thm:lowpassstems} proves that noise attenuation is quadratic with stem kernel size; Theorem~\ref{thm:resolution} yields an analogous gain under anti-aliased downsampling; Theorem~\ref{thm:pooling} shows that average pooling is unbiased with error that falls as the window grows, whereas max pooling is positively biased and, for a uniform signal, its error grows logarithmically; and Theorem~\ref{thm:clip} explains CLIP sensitivity using pixel-space Lipschitz bounds scaling as $1/\sigma_{\min}$, which leads to a $\sim1.91\times$ difference when comparing the \texttt{OPENAI} and \texttt{INCEPTION} constants. These insights provide actionable guidelines for practitioners to enhance the robustness of vision models against Gaussian noise in diverse applications.

\bibliography{iclr2026_conference}

\begin{thebibliography}{46}
\providecommand{\natexlab}[1]{#1}
\providecommand{\url}[1]{\texttt{#1}}
\expandafter\ifx\csname urlstyle\endcsname\relax
  \providecommand{\doi}[1]{doi: #1}\else
  \providecommand{\doi}{doi: \begingroup \urlstyle{rm}\Url}\fi

\bibitem[Anscombe(1948)]{anscombe1948transformation}
Francis~J Anscombe.
\newblock {The transformation of Poisson, binomial and negative-binomial data}.
\newblock \emph{Biometrika}, 35\penalty0 (3/4):\penalty0 246--254, 1948.

\bibitem[Ba et~al.(2016)Ba, Kiros, and Hinton]{DBLP:journals/corr/BaKH16}
Lei~Jimmy Ba, Jamie~Ryan Kiros, and Geoffrey~E. Hinton.
\newblock {Layer Normalization}.
\newblock \emph{CoRR}, abs/1607.06450, 2016.

\bibitem[Bai et~al.(2021)Bai, Mei, Yuille, and Xie]{DBLP:conf/nips/BaiMYX21}
Yutong Bai, Jieru Mei, Alan~L. Yuille, and Cihang Xie.
\newblock {Are Transformers more robust than CNNs?}
\newblock In \emph{NeurIPS}, pp.\  26831--26843, 2021.

\bibitem[Boureau et~al.(2010)Boureau, Ponce, and LeCun]{DBLP:conf/icml/BoureauPL10}
Y{-}Lan Boureau, Jean Ponce, and Yann LeCun.
\newblock {A Theoretical Analysis of Feature Pooling in Visual Recognition}.
\newblock In \emph{{ICML}}, pp.\  111--118, 2010.

\bibitem[Buslaev et~al.(2020)Buslaev, Iglovikov, Khvedchenya, Parinov, Druzhinin, and Kalinin]{DBLP:journals/information/BuslaevIKPDK20}
Alexander Buslaev, Vladimir~I. Iglovikov, Eugene Khvedchenya, Alex Parinov, Mikhail Druzhinin, and Alexandr~A. Kalinin.
\newblock {Albumentations: Fast and Flexible Image Augmentations}.
\newblock \emph{Inf.}, 11\penalty0 (2):\penalty0 125, 2020.

\bibitem[Chen et~al.(2022)Chen, Hsieh, and Gong]{DBLP:conf/iclr/ChenHG22}
Xiangning Chen, Cho{-}Jui Hsieh, and Boqing Gong.
\newblock {When Vision Transformers Outperform ResNets without Pre-training or Strong Data Augmentations}.
\newblock In \emph{{ICLR}}, 2022.

\bibitem[Cubuk et~al.(2020)Cubuk, Zoph, Shlens, and Le]{DBLP:conf/nips/CubukZS020}
Ekin~Dogus Cubuk, Barret Zoph, Jonathon Shlens, and Quoc Le.
\newblock {RandAugment: Practical Automated Data Augmentation with a Reduced Search Space}.
\newblock In \emph{NeurIPS}, 2020.

\bibitem[Deng et~al.(2009)Deng, Dong, Socher, Li, Li, and Fei{-}Fei]{DBLP:conf/cvpr/DengDSLL009}
Jia Deng, Wei Dong, Richard Socher, Li{-}Jia Li, Kai Li, and Li~Fei{-}Fei.
\newblock {ImageNet: {A} large-scale hierarchical image database}.
\newblock In \emph{{CVPR}}, pp.\  248--255, 2009.

\bibitem[Dosovitskiy et~al.(2021)Dosovitskiy, Beyer, Kolesnikov, Weissenborn, Zhai, Unterthiner, Dehghani, Minderer, Heigold, Gelly, Uszkoreit, and Houlsby]{DBLP:conf/iclr/DosovitskiyB0WZ21}
Alexey Dosovitskiy, Lucas Beyer, Alexander Kolesnikov, Dirk Weissenborn, Xiaohua Zhai, Thomas Unterthiner, Mostafa Dehghani, Matthias Minderer, Georg Heigold, Sylvain Gelly, Jakob Uszkoreit, and Neil Houlsby.
\newblock {An Image is Worth 16x16 Words: Transformers for Image Recognition at Scale}.
\newblock In \emph{{ICLR}}, 2021.

\bibitem[Fei{-}Fei et~al.(2007)Fei{-}Fei, Fergus, and Perona]{DBLP:journals/cviu/Fei-FeiFP07}
Li~Fei{-}Fei, Robert Fergus, and Pietro Perona.
\newblock {Learning generative visual models from few training examples: An incremental Bayesian approach tested on 101 object categories}.
\newblock \emph{Comput. Vis. Image Underst.}, 106\penalty0 (1):\penalty0 59--70, 2007.

\bibitem[Gouk et~al.(2021)Gouk, Frank, Pfahringer, and Cree]{DBLP:journals/ml/GoukFPC21}
Henry Gouk, Eibe Frank, Bernhard Pfahringer, and Michael~J. Cree.
\newblock {Regularisation of neural networks by enforcing Lipschitz continuity}.
\newblock \emph{Mach. Learn.}, 110\penalty0 (2):\penalty0 393--416, 2021.

\bibitem[Guo et~al.(2020)Guo, He, He, Lausen, Li, Lin, Shi, Wang, Xie, Zha, Zhang, Zhang, Zhang, Zhang, Zheng, and Zhu]{DBLP:journals/jmlr/GuoHHLLLSWXZZ0Z20}
Jian Guo, He~He, Tong He, Leonard Lausen, Mu~Li, Haibin Lin, Xingjian Shi, Chenguang Wang, Junyuan Xie, Sheng Zha, Aston Zhang, Hang Zhang, Zhi Zhang, Zhongyue Zhang, Shuai Zheng, and Yi~Zhu.
\newblock {GluonCV and GluonNLP: Deep Learning in Computer Vision and Natural Language Processing}.
\newblock \emph{J. Mach. Learn. Res.}, 21:\penalty0 23:1--23:7, 2020.

\bibitem[Hall(1979)]{hall1979rate}
Peter Hall.
\newblock {On the rate of convergence of normal extremes}.
\newblock \emph{Journal of Applied Probability}, 16\penalty0 (2):\penalty0 433--439, 1979.

\bibitem[He et~al.(2016)He, Zhang, Ren, and Sun]{DBLP:conf/cvpr/HeZRS16}
Kaiming He, Xiangyu Zhang, Shaoqing Ren, and Jian Sun.
\newblock {Deep Residual Learning for Image Recognition}.
\newblock In \emph{{CVPR}}, pp.\  770--778, 2016.

\bibitem[He et~al.(2019)He, Zhang, Zhang, Zhang, Xie, and Li]{DBLP:conf/cvpr/HeZ0ZXL19}
Tong He, Zhi Zhang, Hang Zhang, Zhongyue Zhang, Junyuan Xie, and Mu~Li.
\newblock {Bag of Tricks for Image Classification with Convolutional Neural Networks}.
\newblock In \emph{{CVPR}}, pp.\  558--567, 2019.

\bibitem[Hendrycks \& Dietterich(2019)Hendrycks and Dietterich]{DBLP:conf/iclr/HendrycksD19}
Dan Hendrycks and Thomas~G. Dietterich.
\newblock {Benchmarking Neural Network Robustness to Common Corruptions and Perturbations}.
\newblock In \emph{{ICLR}}, 2019.

\bibitem[Hendrycks \& Gimpel(2016)Hendrycks and Gimpel]{DBLP:journals/corr/HendrycksG16}
Dan Hendrycks and Kevin Gimpel.
\newblock {Bridging Nonlinearities and Stochastic Regularizers with Gaussian Error Linear Units}.
\newblock \emph{CoRR}, abs/1606.08415, 2016.

\bibitem[Huang et~al.(2016)Huang, Sun, Liu, Sedra, and Weinberger]{DBLP:conf/eccv/HuangSLSW16}
Gao Huang, Yu~Sun, Zhuang Liu, Daniel Sedra, and Kilian~Q. Weinberger.
\newblock {Deep Networks with Stochastic Depth}.
\newblock In \emph{{ECCV} {(4)}}, volume 9908, pp.\  646--661, 2016.

\bibitem[Krause et~al.(2013)Krause, Stark, Deng, and Fei{-}Fei]{DBLP:conf/iccvw/Krause0DF13}
Jonathan Krause, Michael Stark, Jia Deng, and Li~Fei{-}Fei.
\newblock {3D Object Representations for Fine-Grained Categorization}.
\newblock In \emph{{ICCV} Workshops}, pp.\  554--561, 2013.

\bibitem[Loshchilov \& Hutter(2017)Loshchilov and Hutter]{DBLP:conf/iclr/LoshchilovH17}
Ilya Loshchilov and Frank Hutter.
\newblock {{SGDR:} Stochastic Gradient Descent with Warm Restarts}.
\newblock In \emph{{ICLR}}, 2017.

\bibitem[Loshchilov \& Hutter(2019)Loshchilov and Hutter]{DBLP:conf/iclr/LoshchilovH19}
Ilya Loshchilov and Frank Hutter.
\newblock {Decoupled Weight Decay Regularization}.
\newblock In \emph{{ICLR}}, 2019.

\bibitem[Maji et~al.(2013)Maji, Rahtu, Kannala, Blaschko, and Vedaldi]{DBLP:journals/corr/MajiRKBV13}
Subhransu Maji, Esa Rahtu, Juho Kannala, Matthew~B. Blaschko, and Andrea Vedaldi.
\newblock {Fine-Grained Visual Classification of Aircraft}.
\newblock \emph{CoRR}, abs/1306.5151, 2013.

\bibitem[Matoba et~al.(2023)Matoba, Dimitriadis, and Fleuret]{DBLP:journals/tmlr/MatobaDF23}
Kyle Matoba, Nikolaos Dimitriadis, and Fran{\c{c}}ois Fleuret.
\newblock {Benefits of Max Pooling in Neural Networks: Theoretical and Experimental Evidence}.
\newblock \emph{Trans. Mach. Learn. Res.}, 2023, 2023.

\bibitem[Naseer et~al.(2021)Naseer, Ranasinghe, Khan, Hayat, Khan, and Yang]{DBLP:conf/nips/NaseerRKHKY21}
Muzammal Naseer, Kanchana Ranasinghe, Salman Khan, Munawar Hayat, Fahad~Shahbaz Khan, and Ming{-}Hsuan Yang.
\newblock {Intriguing Properties of Vision Transformers}.
\newblock In \emph{NeurIPS}, pp.\  23296--23308, 2021.

\bibitem[Oppenheim(1999)]{oppenheim1999discrete}
Alan~V Oppenheim.
\newblock \emph{{Discrete-time signal processing}}.
\newblock 1999.

\bibitem[Parkhi et~al.(2012)Parkhi, Vedaldi, Zisserman, and Jawahar]{DBLP:conf/cvpr/ParkhiVZJ12}
Omkar~M. Parkhi, Andrea Vedaldi, Andrew Zisserman, and C.~V. Jawahar.
\newblock {Cats and dogs}.
\newblock In \emph{{CVPR}}, pp.\  3498--3505, 2012.

\bibitem[Paul \& Chen(2022)Paul and Chen]{DBLP:conf/aaai/PaulC22}
Sayak Paul and Pin{-}Yu Chen.
\newblock {Vision Transformers Are Robust Learners}.
\newblock In \emph{{AAAI}}, pp.\  2071--2081, 2022.

\bibitem[Qian et~al.(2021)Qian, Shao, Zhu, Li, and Jia]{DBLP:conf/nips/QianSZLJ21}
Shengju Qian, Hao Shao, Yi~Zhu, Mu~Li, and Jiaya Jia.
\newblock {Blending Anti-Aliasing into Vision Transformer}.
\newblock In \emph{NeurIPS}, pp.\  5416--5429, 2021.

\bibitem[Radford et~al.(2021)Radford, Kim, Hallacy, Ramesh, Goh, Agarwal, Sastry, Askell, Mishkin, Clark, Krueger, and Sutskever]{DBLP:conf/icml/RadfordKHRGASAM21}
Alec Radford, Jong~Wook Kim, Chris Hallacy, Aditya Ramesh, Gabriel Goh, Sandhini Agarwal, Girish Sastry, Amanda Askell, Pamela Mishkin, Jack Clark, Gretchen Krueger, and Ilya Sutskever.
\newblock {Learning Transferable Visual Models From Natural Language Supervision}.
\newblock In \emph{{ICML}}, volume 139, pp.\  8748--8763, 2021.

\bibitem[Recht et~al.(2019)Recht, Roelofs, Schmidt, and Shankar]{DBLP:conf/icml/RechtRSS19}
Benjamin Recht, Rebecca Roelofs, Ludwig Schmidt, and Vaishaal Shankar.
\newblock {Do ImageNet Classifiers Generalize to ImageNet?}
\newblock In \emph{{ICML}}, volume~97, pp.\  5389--5400, 2019.

\bibitem[Shu et~al.(2023)Shu, Guo, Wu, Wang, Wang, and Long]{DBLP:conf/icml/ShuG0W0L23}
Yang Shu, Xingzhuo Guo, Jialong Wu, Ximei Wang, Jianmin Wang, and Mingsheng Long.
\newblock {CLIPood: Generalizing {CLIP} to Out-of-Distributions}.
\newblock In \emph{{ICML}}, volume 202, pp.\  31716--31731, 2023.

\bibitem[Steiner et~al.(2022)Steiner, Kolesnikov, Zhai, Wightman, Uszkoreit, and Beyer]{DBLP:journals/tmlr/SteinerKZWUB22}
Andreas Steiner, Alexander Kolesnikov, Xiaohua Zhai, Ross Wightman, Jakob Uszkoreit, and Lucas Beyer.
\newblock {How to train your ViT? Data, Augmentation, and Regularization in Vision Transformers}.
\newblock \emph{Trans. Mach. Learn. Res.}, 2022, 2022.

\bibitem[Szegedy et~al.(2016)Szegedy, Vanhoucke, Ioffe, Shlens, and Wojna]{DBLP:conf/cvpr/SzegedyVISW16}
Christian Szegedy, Vincent Vanhoucke, Sergey Ioffe, Jonathon Shlens, and Zbigniew Wojna.
\newblock {Rethinking the Inception Architecture for Computer Vision}.
\newblock In \emph{{CVPR}}, pp.\  2818--2826, 2016.

\bibitem[Tan \& Le(2019)Tan and Le]{DBLP:conf/icml/TanL19}
Mingxing Tan and Quoc~V. Le.
\newblock {EfficientNet: Rethinking Model Scaling for Convolutional Neural Networks}.
\newblock In \emph{{ICML}}, volume~97, pp.\  6105--6114, 2019.

\bibitem[Touvron et~al.(2021)Touvron, Cord, Douze, Massa, Sablayrolles, and J{\'{e}}gou]{DBLP:conf/icml/TouvronCDMSJ21}
Hugo Touvron, Matthieu Cord, Matthijs Douze, Francisco Massa, Alexandre Sablayrolles, and Herv{\'{e}} J{\'{e}}gou.
\newblock {Training data-efficient image transformers {\&} distillation through attention}.
\newblock In \emph{{ICML}}, volume 139, pp.\  10347--10357, 2021.

\bibitem[Tsuzuku et~al.(2018)Tsuzuku, Sato, and Sugiyama]{DBLP:conf/nips/TsuzukuSS18}
Yusuke Tsuzuku, Issei Sato, and Masashi Sugiyama.
\newblock {Lipschitz-Margin Training: Scalable Certification of Perturbation Invariance for Deep Neural Networks}.
\newblock In \emph{NeurIPS}, pp.\  6542--6551, 2018.

\bibitem[Vershynin(2018)]{vershynin2018introduction}
Roman Vershynin.
\newblock {High-Dimensional Probability: An Introduction with Applications in Data Science}, 2018.

\bibitem[Virmaux \& Scaman(2018)Virmaux and Scaman]{DBLP:conf/nips/VirmauxS18}
Aladin Virmaux and Kevin Scaman.
\newblock {Lipschitz regularity of deep neural networks: analysis and efficient estimation}.
\newblock In \emph{NeurIPS}, pp.\  3839--3848, 2018.

\bibitem[Wah et~al.(2011)Wah, Branson, Welinder, Perona, and Belongie]{wah2011caltech}
Catherine Wah, Steve Branson, Peter Welinder, Pietro Perona, and Serge Belongie.
\newblock {The caltech-ucsd birds-200-2011 dataset}.
\newblock 2011.

\bibitem[Wightman(2019)]{rw2019timm}
Ross Wightman.
\newblock {PyTorch Image Models}.
\newblock \url{https://github.com/rwightman/pytorch-image-models}, 2019.

\bibitem[Wightman et~al.(2021)Wightman, Touvron, and J{\'{e}}gou]{DBLP:journals/corr/abs-2110-00476}
Ross Wightman, Hugo Touvron, and Herv{\'{e}} J{\'{e}}gou.
\newblock {ResNet strikes back: An improved training procedure in timm}.
\newblock \emph{CoRR}, abs/2110.00476, 2021.

\bibitem[Wortsman et~al.(2022)Wortsman, Ilharco, Kim, Li, Kornblith, Roelofs, Lopes, Hajishirzi, Farhadi, Namkoong, and Schmidt]{DBLP:conf/cvpr/WortsmanIKLKRLH22}
Mitchell Wortsman, Gabriel Ilharco, Jong~Wook Kim, Mike Li, Simon Kornblith, Rebecca Roelofs, Raphael~Gontijo Lopes, Hannaneh Hajishirzi, Ali Farhadi, Hongseok Namkoong, and Ludwig Schmidt.
\newblock {Robust fine-tuning of zero-shot models}.
\newblock In \emph{{CVPR}}, pp.\  7949--7961, 2022.

\bibitem[Yun et~al.(2019)Yun, Han, Chun, Oh, Yoo, and Choe]{DBLP:conf/iccv/YunHCOYC19}
Sangdoo Yun, Dongyoon Han, Sanghyuk Chun, Seong~Joon Oh, Youngjoon Yoo, and Junsuk Choe.
\newblock {CutMix: Regularization Strategy to Train Strong Classifiers With Localizable Features}.
\newblock In \emph{{ICCV}}, pp.\  6022--6031, 2019.

\bibitem[Zhang(2019)]{DBLP:conf/icml/Zhang19}
Richard Zhang.
\newblock {Making Convolutional Networks Shift-Invariant Again}.
\newblock In \emph{{ICML}}, volume~97, pp.\  7324--7334, 2019.

\bibitem[Zhong et~al.(2020)Zhong, Zheng, Kang, Li, and Yang]{DBLP:conf/aaai/Zhong0KL020}
Zhun Zhong, Liang Zheng, Guoliang Kang, Shaozi Li, and Yi~Yang.
\newblock {Random Erasing Data Augmentation}.
\newblock In \emph{{AAAI}}, pp.\  13001--13008, 2020.

\bibitem[Zou et~al.(2023)Zou, Xiao, Yu, Li, and Lee]{DBLP:journals/ijcv/ZouXYLL23}
Xueyan Zou, Fanyi Xiao, Zhiding Yu, Yuheng Li, and Yong~Jae Lee.
\newblock {Delving Deeper into Anti-Aliasing in ConvNets}.
\newblock \emph{Int. J. Comput. Vis.}, 131\penalty0 (1):\penalty0 67--81, 2023.

\end{thebibliography}
\bibliographystyle{iclr2026_conference}

\newpage
\appendix
\startcontents[apx]
\section*{Appendix Table of Contents}
\printcontents[apx]{l}{1}{\setcounter{tocdepth}{2}}

\section{Proofs for Theorems~\ref{thm:lowpassstems} and~\ref{thm:resolution}}
\label{app:proofs}
Here, we provide proofs of the quadratic noise-decay results in \secref{sec:ksrs-theory}.

\subsection{Conventions and assumptions}
\label{app:standing}

\paragraph{DFT convention and Parseval.} For $u\in\mathbb{R}^{H\times W}$ with discrete Fourier transform (DFT) $\widehat u$ on the frequency grid $\Omega$, we use the Parseval identity
\begin{equation}\label{eq:parseval-app}
	\frac{1}{HW}\sum_{\boldsymbol\omega\in\Omega}\bigl|\widehat u(\boldsymbol\omega)\bigr|^2
	= \sum_{p\in\{1,\dots,H\}\times\{1,\dots,W\}} |u(p)|^2.
\end{equation}
We write $\varepsilon \coloneqq 2\pi/\max\{H,W\}$ for the infrared cutoff.

\paragraph{Filter family.} For $k\ge3$, let $K_k\in\mathbb{R}^{k\times k}$ denote the linear, shift-invariant stem kernel with DFT $\widehat K_k$. We assume only the following low-pass envelope; the same assumption applies to $K_{g(s)}$ when used as the anti-aliasing filter at scale $g(s)$:
\begin{itemize}[leftmargin=1.25em]
	\item \textbf{(A\textsubscript{roll})} (Radial low-pass envelope at scale $1/k$) There exist $\beta,\delta>0$ such that, for all frequencies $\boldsymbol\omega$,
	      \[
		      |\widehat K_k(\boldsymbol\omega)| \le \phi_k(\|\boldsymbol\omega\|),
		      \qquad \phi_k(r) \coloneqq (1+\beta k r)^{-1-\delta},
	      \]
	      where $\phi_k$ is nonincreasing in $r$.
\end{itemize}
This assumption provides a monotone radial upper envelope sufficient for establishing our upper bounds: When estimating $\frac{1}{HW}\sum_{\boldsymbol\omega}|\widehat K_k(\boldsymbol\omega)|^2$, we first dominate $|\widehat K_k|^2$ by $\phi_k^2$ and then apply the sum-integral comparison in \eqref{eq:sum-to-int}.

\paragraph{Noise model and gains.} Let $\eta\sim\mathcal N(0,\sigma^2 I)$ be spatially white Gaussian noise. The per-pixel noise gain of the stem kernel is
\begin{equation}\label{eq:gamma-def-app}
	\gamma(k) \coloneqq \frac{\mathbb{E} \bigl[\|K_k*\eta\|_2^2\bigr]}{\sigma^2 HW}
	\stackrel{\eqref{eq:parseval-app}}=\frac{1}{HW}\sum_{\boldsymbol\omega} \bigl|\widehat K_k(\boldsymbol\omega)\bigr|^2
	= \|K_k\|_{F}^{2}.
\end{equation}
For anti-aliased downsampling with a factor $s\ge1$, we define
\begin{equation}
	D_s \coloneqq (\Downarrow_s)\circ K_{g(s)},\qquad c_1 s \le g(s) \le c_2 s,
\end{equation}
and its per-output-pixel noise gain
\begin{equation}\label{eq:gamma-down-def-app}
	\gamma_{\downarrow}(s) \coloneqq \frac{\mathbb{E} \bigl[\|D_s\eta\|_2^2\bigr]}{\sigma^2 HW/s^{2}}.
\end{equation}

\paragraph{Radial sum-integral comparison.} Let $\Omega$ be the $H\times W$ DFT grid with spacing $\varepsilon$, and let $g:[\varepsilon,\pi]\to\mathbb{R}_{\ge0}$ be radially nonincreasing. We partition $\Omega$ into annuli $\mathcal A_j\coloneqq\{\boldsymbol\omega: j\varepsilon\le\|\boldsymbol\omega\|<(j+1)\varepsilon\}$. Because each grid point occupies an area $\asymp \varepsilon^2$ and the annulus area is $2\pi r \varepsilon$ up to boundary effects, there exist absolute lattice constants $c_1,c_2>0$---independent of $H,W,k,s$---with
\[
	c_1 HW (2\pi j\varepsilon)\varepsilon \le |\mathcal A_j|\le c_2 HW (2\pi (j+1)\varepsilon)\varepsilon.
\]
By monotonicity, $g((j+1)\varepsilon) |\mathcal A_j|\le\sum_{\boldsymbol\omega\in\mathcal A_j} g(\|\boldsymbol\omega\|)\le g(j\varepsilon) |\mathcal A_j|$. Summing over $j$ and dividing by $HW$ turns the lattice sum into upper and lower Riemann sums for $r\mapsto 2\pi r g(r)$ with mesh $\varepsilon$, yielding absolute constants $A_1,A_2>0$ such that
\begin{equation}\label{eq:sum-to-int}
	A_1 \int_{\varepsilon}^{\pi} r g(r) dr \le \frac{1}{HW}\sum_{\boldsymbol\omega\in\Omega} g(\|\boldsymbol\omega\|) \le A_2 \int_{\varepsilon}^{\pi} r g(r) dr.
\end{equation}
As $\varepsilon\to 0$, both bounds converge to the same limit; for finite grids, $A_1,A_2$ absorb edge discrepancies and remain independent of the kernel scale $k$ or downsampling factor $s$.

\subsection{Proof of Theorem~\ref{thm:lowpassstems} (Quadratic decay in stem kernel size)}
\label{app:proofquadkernel}

\begin{proof}
	By \eqref{eq:gamma-def-app}, \eqref{eq:sum-to-int}, and \textup{(A\textsubscript{roll})}, we have
	\[
		\gamma(k) \lesssim \int_{\varepsilon}^{\pi} r \bigl|\widehat K_k(r)\bigr|^2 dr
		\le \int_{\varepsilon}^{\pi} r \bigl(1+\beta k r\bigr)^{-2-2\delta} dr.
	\]
	Let $u=1+\beta k r$. Then $r=(u-1)/(\beta k)$ and $dr=du/(\beta k)$, so
	\[
		\int_{\varepsilon}^{\pi} r (1+\beta k r)^{-2-2\delta} dr
		=\frac{1}{\beta^2 k^2}\int_{1+\beta k\varepsilon}^{1+\beta k\pi} \frac{u-1}{u^{2+2\delta}} du
		\le \frac{1}{\beta^2 k^2}\int_{1}^{\infty} \frac{u-1}{u^{2+2\delta}} du
		=\frac{C}{k^2},
	\]
	for a finite constant $C=C(\beta,\delta)$. Hence $\gamma(k)\le C'/k^2$ for some $C'$ independent of $k$.
\end{proof}

\subsection{Proof of Theorem~\ref{thm:resolution} (Quadratic decay under anti-aliased downsampling)}
\label{app:proofquaddown}
We first state the following identity for white noise.

\begin{lemma}[Per-output-pixel gain identity]\label{lem:down-id}
	For $D_s$ defined in \eqref{eq:Ds-def} and white noise $\eta\sim\mathcal N(0,\sigma^2 I)$,
	\[
		\gamma_{\downarrow}(s) = \|K_{g(s)}\|_{F}^{2}.
	\]
\end{lemma}

\begin{proof}
	Stationarity of white noise and \eqref{eq:parseval-app} give
	\[
		\mathbb{E} \bigl[\|K_{g(s)}*\eta\|_2^2\bigr] = (HW) \sigma^2 \|K_{g(s)}\|_F^2.
	\]
	Downsampling by $s$ keeps every $s$-th sample along each axis: The retained samples all have equal variance as the original, pre-downsampled field. Therefore,
	\[
		\mathbb{E} \bigl[\|D_s\eta\|_2^2\bigr] = \frac{HW}{s^2} \sigma^2 \|K_{g(s)}\|_F^2,
	\]
	and the normalization in \eqref{eq:gamma-down-def-app} yields $\gamma_{\downarrow}(s)=\|K_{g(s)}\|_F^2$.
\end{proof}

\begin{proof}[Proof of Theorem~\ref{thm:resolution}]
	By Lemma~\ref{lem:down-id}, $\gamma_{\downarrow}(s)=\|K_{g(s)}\|_F^2$. Applying Theorem~\ref{thm:lowpassstems} with kernel size $k=g(s)$ gives
	\[
		\gamma_{\downarrow}(s) \le \frac{C}{g(s)^2} \le \frac{C}{(c_1 s)^2} = \frac{C'}{s^2},
	\]
	with $C'=C/c_1^2$ independent of $s$.
\end{proof}

\section{Proof of Theorem~\ref{thm:pooling} (Average and max poolings under Gaussian Noise)}
\label{app:proof-avg-vs-max}
Consider a pooling window of size \(k\ge2\) in a single channel. Let the clean activations be \(S=(S_1,\dots,S_k)\in\mathbb{R}^k\) and let the observation be \(S+\eta\), where \(\eta=(\eta_1,\dots,\eta_k)\stackrel{\text{i.i.d.}}{\sim}\mathcal N(0,\sigma^2)\). We define
\[
	X_{\mathrm{avg}} \coloneqq \frac{1}{k}\sum_{i=1}^k (S_i+\eta_i),
	\qquad
	X_{\mathrm{max}} \coloneqq \max_{1\le i\le k}(S_i+\eta_i),
\]
and their clean counterparts \(S_{\mathrm{avg}}=\frac{1}{k}\sum_i S_i\), \(S_{\mathrm{max}}=\max_i S_i\). Let the errors be \(\delta_{\mathrm{avg}} \coloneqq X_{\mathrm{avg}}-S_{\mathrm{avg}}\), \(\delta_{\mathrm{max}} \coloneqq X_{\mathrm{max}}-S_{\mathrm{max}}\). Write the order statistics \(S_{(1)}\ge\cdots\ge S_{(k)}\), define the gap \(\Delta\coloneqq S_{(1)}-S_{(2)}\ge0\), and the standardized gap \(z\coloneqq \Delta/\sigma\). We use \(Z_i\stackrel{\text{i.i.d.}}{\sim}\mathcal N(0,1)\), \(M_k\coloneqq \max_{1\le i\le k} Z_i\), and \(A_k\coloneqq \max_{1\le i\le k}|Z_i|\).

\begin{proof}[Proof of \textup{(i)}]
	By definition, \(\delta_{\mathrm{avg}}=\frac{1}{k}\sum_{i=1}^k \eta_i\). Hence
	\[
		\mathbb{E} [\delta_{\mathrm{avg}}]=\frac1k\sum_i \mathbb{E} [\eta_i]=0,
		\qquad
		\Var[\delta_{\mathrm{avg}}]=\frac{1}{k^2}\sum_i \Var [\eta_i]=\frac{\sigma^2}{k}.
	\]
	This part requires only i.i.d. zero-mean noise with variance \(\sigma^2\).
\end{proof}

\begin{proof}[Proof of \textup{(ii)}]
	(Positive bias) Let \(i^\star\in\arg\max_i S_i\). Then \(X_{\max}\ge S_{i^\star}+\eta_{i^\star}\). Taking expectations and using \(\mathbb{E} [\eta_{i^\star}]=0\) yields
	\[
		\mathbb{E} [\delta_{\max}] =\mathbb{E} [X_{\max}-S_{\max}]\ge 0.
	\]

	\begin{figure}[h!]
		\centering
		\includegraphics[width=.99\linewidth]{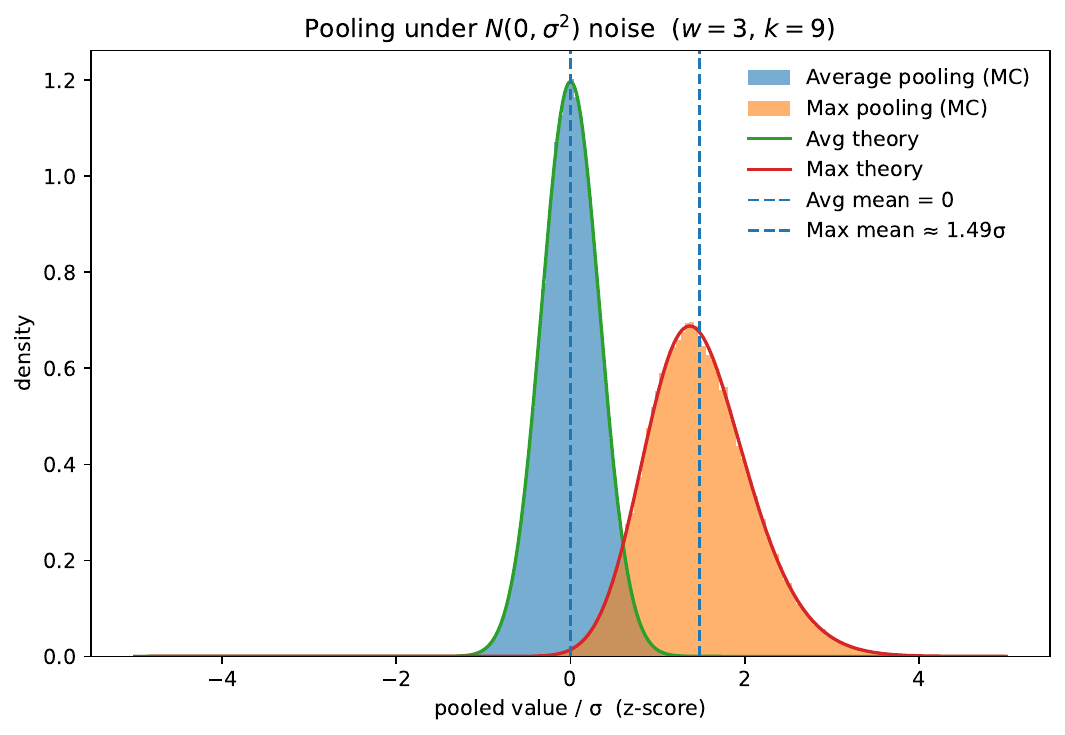}
		\caption{Illustration of positive bias introduced by max pooling}
		\label{fig:poolbias}
	\end{figure}

	(Uniform-signal case) If \(S_1=\cdots=S_k\), translate so \(S_i\equiv0\). Then \(\delta_{\max}=\max_i\eta_i=\sigma M_k\) and
	\[
		\mathbb{E} [\delta_{\max}^2]=\sigma^2 \mathbb{E} [M_k^2].
	\]
	Classical Gaussian extreme-value asymptotics \citep{hall1979rate} give
	\[
		\mathbb{E} [M_k]=\sqrt{2\log k}-\frac{\log \log k+\log(4\pi)}{2\sqrt{2\log k}}+o \left((\log k)^{-1/2}\right),
		\quad
		\Var[M_k]=\frac{\pi^2}{12\log k}+o \left((\log k)^{-1}\right),
	\]
	hence
	\[
		\mathbb{E} [M_k^2]=\Var[M_k]+\big(\mathbb{E} [M_k]\big)^2
		=2\log k-\log \log k-\log(4\pi)+o(1),
	\]
	Because $\delta_{\max}=\sigma M_k$, we have
	\[
		\mathbb{E}[\delta_{\max}^2] = \sigma^2 [\mathbb{E} M_k^2] = \sigma^2 (2\log k-\log\log k-\log(4\pi)+o(1)) = \Theta(\sigma^2\log k),
	\]
	so the MSE scales as \(\Theta(\sigma^2\log k)\).

	(General case) For any realization,
	\[
		\big|\delta_{\max}\big|=\big|\max_i(S_i+\eta_i)-\max_i S_i\big|\le \max_i|\eta_i|=\sigma A_k.
	\]
	Hence
	\[
		\mathbb{E} [\delta_{\max}^2]\le \sigma^2 \mathbb{E}[A_k^2].
	\]
	We now bound \(\mathbb{E}[A_k^2]\) explicitly.

	\begin{lemma}\label{lem:Ak2}
		For \(A_k=\max_{1\le i\le k}|Z_i|\) with \(Z_i\stackrel{\text{i.i.d.}}{\sim}\mathcal N(0,1)\), we have \( \mathbb{E} [A_k^2]\le 2\log(2k)+2\).
	\end{lemma}

	\begin{proof}[Proof of Lemma~\ref{lem:Ak2}]
		For \(t\ge0\), \(\Pr(A_k\ge t)\le \sum_{i=1}^k \Pr(|Z_i|\ge t)\le 2k e^{-t^2/2}\), where the last step uses the union bound and the standard Gaussian tail estimate \(\Pr(|Z|\ge t)\le 2e^{-t^2/2}\) for \(Z\sim\mathcal N(0,1)\); see, \eg, \citet{vershynin2018introduction}. Using \(\mathbb{E} [X^2]=\int_0^\infty 2t \Pr(X\ge t) dt\) for a nonnegative \(X\) and splitting at \(t_0\coloneqq \sqrt{2\log(2k)}\),
		\begin{align*}
			\mathbb{E} [A_k^2] & = \int_0^{t_0} 2t \Pr(A_k\ge t) dt + \int_{t_0}^{\infty} 2t \Pr(A_k\ge t) dt     \\
			                   & \le t_0^2 + \int_{t_0}^{\infty} 4k t e^{-t^2/2} dt = 2\log(2k)+ 4k e^{-t_0^2/2}.
		\end{align*}
		Because \(e^{-t_0^2/2}=e^{-\log(2k)}=1/(2k)\), the last term equals \(2\), proving the claim.
	\end{proof}

	By Lemma~\ref{lem:Ak2},
	\[
		\mathbb{E} [\delta_{\max}^2]\le \sigma^2\big(2\log(2k)+2\big).
	\]
\end{proof}

\begin{proof}[Proof of \textup{(iii)}]
	Let \(T_{\mathrm{avg}}(n)=\tfrac1k\sum_i n_i\). By Cauchy--Schwarz, \(|T_{\mathrm{avg}}(n)|\le \|n\|_2 \|k^{-1}(1,\dots,1)\|_2=\|n\|_2/\sqrt{k}\), so \(\|T_{\mathrm{avg}}\|_{\ell_2\to\ell_2}=k^{-1/2}\), tight for constant \(n\). For max, for any \(a,b\), \(|\max_i a_i-\max_i b_i|\le \|a-b\|_\infty\le \|a-b\|_2\), hence \(\|T_{\mathrm{max}}\|_{\ell_2\to\ell_2}\le 1\), tight for a one-hot \(n\).
\end{proof}

\begin{proof}[Proof of \textup{(iv)}]
	Translate so \(S_{(1)}=0\) and \(S_i\le -\Delta\) for \(i\ge2\). Let \(\mathcal S\) be the switch event that some \(j\ge2\) overtakes the top index after noise:
	\[
		\mathcal S \coloneqq \{\exists j\ge2: \eta_j-\Delta \ge \eta_1\}
		= \{\exists j\ge2: Z_j-Z_1 \ge z\}.
	\]
	Because $Z_j$ and $Z_1$ are independent standard normals, we have $Z_j-Z_1\sim\mathcal N(0,2)$; hence, by a union bound \(\Pr(\mathcal S)\le (k-1)\Pr(\mathcal N(0,2)\ge z)\le (k-1)e^{-z^2/4}\to0\) as \(z\to\infty\). On \(\mathcal S^c\), \(X_{\max}=S_{(1)}+\eta_1=\eta_1\), so \(\delta_{\max}^2=\eta_1^2\). Dominated convergence then gives \(\mathbb{E} [\delta_{\max}^2]\to \mathbb{E} [\eta_1^2]=\sigma^2\).
\end{proof}

\begin{figure}[h!]
	\centering
	\includegraphics[width=.99\linewidth]{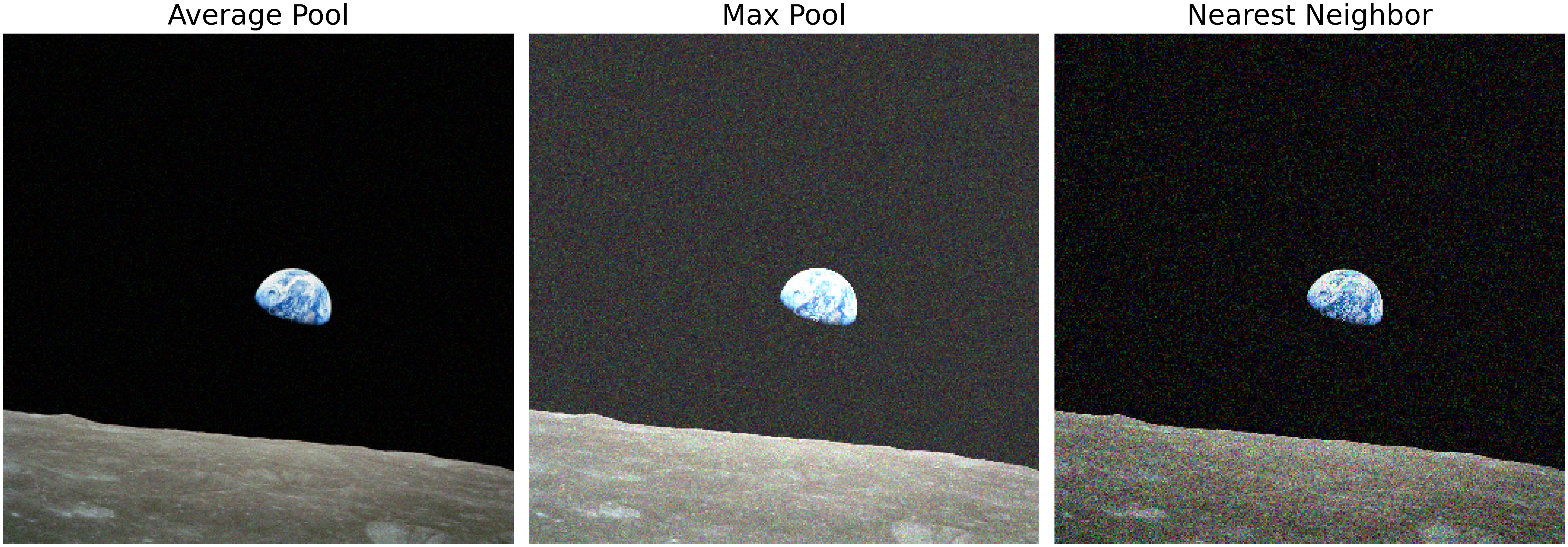}
	\caption{Examples of pooling outputs from a noisy image using average, max, and nearest neighbor}
	\label{fig:poolex}
\end{figure}

\section{Additional Experimental Results}
\label{app:additional}

\figref{fig:ksrsval} shows the accuracy on the validation set for the controlled experiments on kernel size and resolution.

\begin{figure}[t!]
	\centering
	\begin{subfigure}{.5\textwidth}
		\centering
		\includegraphics[width=.99\linewidth]{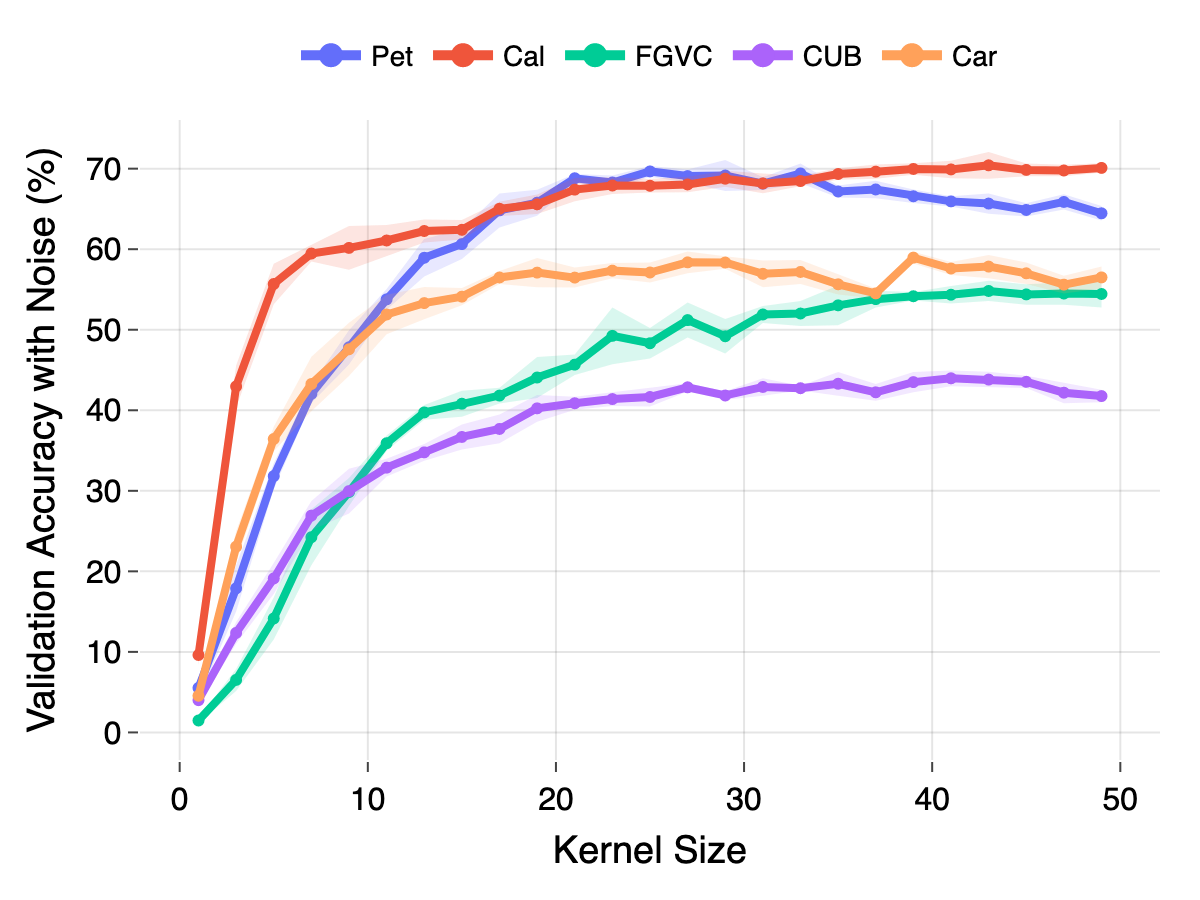}
	\end{subfigure}%
	\begin{subfigure}{.5\textwidth}
		\centering
		\includegraphics[width=.99\linewidth]{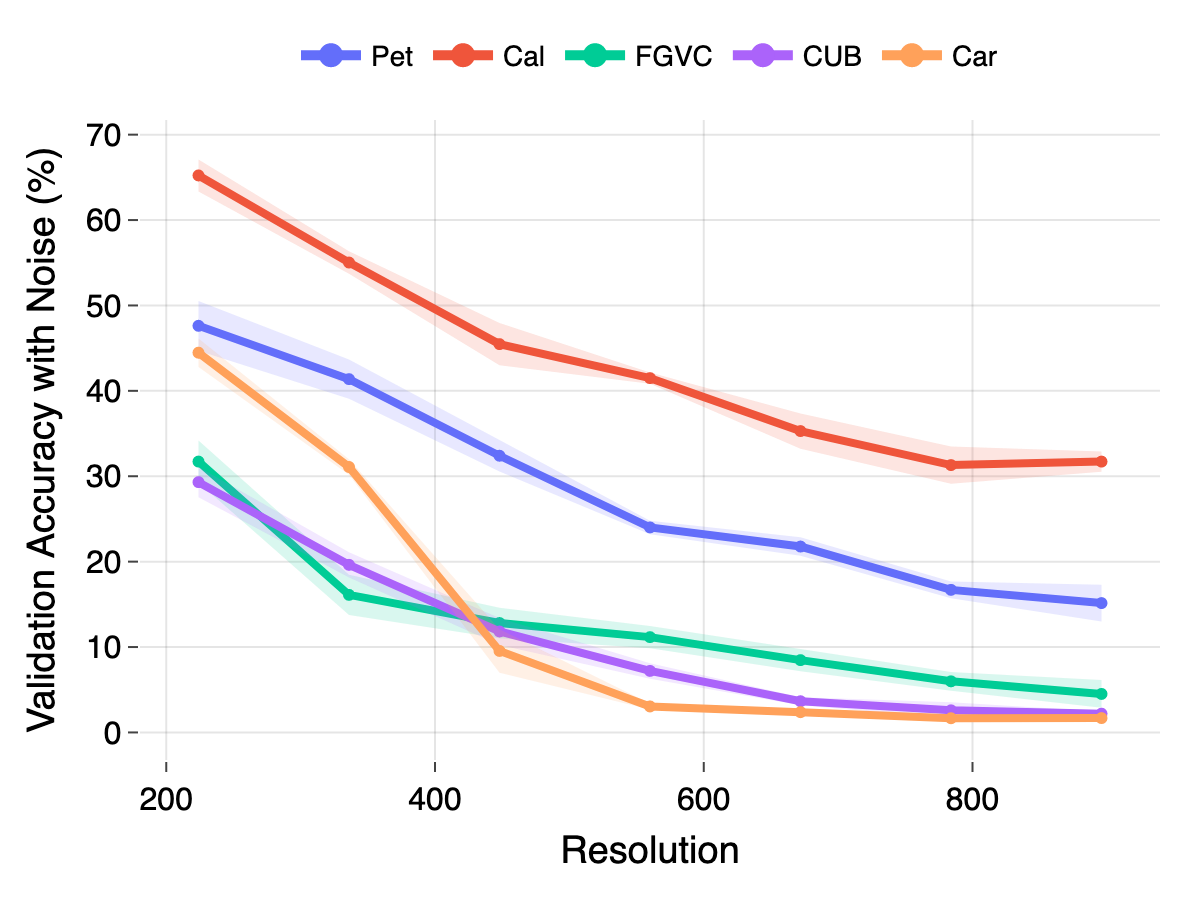}
	\end{subfigure}
	\caption{The results on the validation set}
	\label{fig:ksrsval}
\end{figure}

We also report additional results for ResNet-50-D (\tabref{tab:r50d}) and ResNet-101-D (\tabref{tab:r101d}) for different choices of pooling.

\begin{table}[t!]
	\caption{Classification accuracy comparing different poolings, using ResNet-50-D}
	\label{tab:r50d}
	\begin{center}
		\begin{tabular}{ll|ccc}
			\toprule
			\textbf{Dataset                          } & \textbf{Model              } & \textbf{MaxPool                               } & \textbf{NNPool     } & \textbf{AvgPool    } \\
			\midrule
			\multirow{4}{*}{Oxford-IIIT Pet}           & Val. Acc.                    & 87.7                            (0.6)           & 87.6 (0.4)           & 86.7 (0.5)           \\
			                                           & Test Acc.                    & 85.3                            (0.8)           & 84.7 (0.6)           & 84.8 (0.9)           \\
			                                           & Val. Acc. w/ Noise           & 48.3                            (2.2)           & 46.3 (1.9)           & 54.0 (3.7)           \\
			                                           & Test Acc. w/ Noise           & 47.8                            (1.3)           & 45.2 (2.3)           & 53.6 (2.7)           \\
			\midrule
			\multirow{4}{*}{Caltech-101}               & Val. Acc.                    & 81.3                            (0.7)           & 82.4 (1.1)           & 81.7 (0.5)           \\
			                                           & Test Acc.                    & 80.5                            (0.3)           & 80.7 (0.4)           & 81.6 (0.7)           \\
			                                           & Val. Acc. w/ Noise           & 61.1                            (1.5)           & 60.3 (1.4)           & 62.7 (1.4)           \\
			                                           & Test Acc. w/ Noise           & 59.8                            (1.7)           & 58.3 (1.3)           & 61.6 (1.3)           \\
			\midrule
			\multirow{4}{*}{FGVC-Aircraft}             & Val. Acc.                    & 68.1                            (0.2)           & 67.7 (0.8)           & 69.0 (0.7)           \\
			                                           & Test Acc.                    & 68.8                            (1.1)           & 68.3 (1.5)           & 69.6 (0.3)           \\
			                                           & Val. Acc. w/ Noise           & 27.7                            (1.6)           & 24.8 (1.8)           & 42.9 (1.7)           \\
			                                           & Test Acc. w/ Noise           & 31.5                            (2.1)           & 26.9 (0.8)           & 44.8 (1.1)           \\
			\midrule
			\multirow{4}{*}{\shortstack[l]{Caltech-UCSD                                                                                                                               \\ Birds-200-2011}} & Val. Acc.             & 69.8                            (0.7)                 & 69.8 (0.4)                        & 69.3 (1.1)                         \\
			                                           & Test Acc.                    & 67.3                            (0.4)           & 66.4 (0.6)           & 65.9 (0.4)           \\
			                                           & Val. Acc. w/ Noise           & 26.8                            (0.6)           & 28.7 (1.7)           & 31.8 (1.6)           \\
			                                           & Test Acc. w/ Noise           & 26.0                            (0.7)           & 27.4 (1.2)           & 31.1 (2.1)           \\
			\midrule
			\multirow{4}{*}{Stanford Cars}             & Val. Acc.                    & 86.5                            (0.5)           & 85.7 (0.5)           & 84.9 (0.2)           \\
			                                           & Test Acc.                    & 84.8                            (0.2)           & 83.6 (0.3)           & 83.2 (0.3)           \\
			                                           & Val. Acc. w/ Noise           & 56.0                            (0.5)           & 53.6 (1.6)           & 56.8 (2.2)           \\
			                                           & Test Acc. w/ Noise           & 54.8                            (1.5)           & 51.6 (1.5)           & 55.3 (2.0)           \\
			\bottomrule
		\end{tabular}
	\end{center}
\end{table}

\begin{table}[t!]
	\caption{Classification accuracy comparing different poolings, using ResNet-101-D}
	\label{tab:r101d}
	\begin{center}
		\begin{tabular}{ll|ccc}
			\toprule
			\textbf{Dataset                          } & \textbf{Model              } & \textbf{MaxPool    } & \textbf{NNPool     } & \textbf{AvgPool    } \\
			\midrule
			\multirow{4}{*}{Oxford-IIIT Pet}           & Val. Acc.                    & 87.0 (0.5)           & 86.5 (0.8)           & 86.2 (0.3)           \\
			                                           & Test Acc.                    & 84.8 (0.6)           & 84.4 (0.7)           & 84.3 (0.6)           \\
			                                           & Val. Acc. w/ Noise           & 52.3 (1.9)           & 51.0 (1.5)           & 56.4 (2.1)           \\
			                                           & Test Acc. w/ Noise           & 51.0 (1.3)           & 49.2 (1.4)           & 56.3 (2.4)           \\
			\midrule
			\multirow{4}{*}{Caltech-101}               & Val. Acc.                    & 82.0 (0.9)           & 82.9 (0.6)           & 82.9 (0.5)           \\
			                                           & Test Acc.                    & 80.6 (0.4)           & 80.7 (0.9)           & 81.2 (0.4)           \\
			                                           & Val. Acc. w/ Noise           & 63.4 (1.9)           & 63.7 (1.3)           & 64.8 (1.1)           \\
			                                           & Test Acc. w/ Noise           & 62.1 (1.5)           & 61.6 (1.7)           & 63.7 (1.4)           \\
			\midrule
			\multirow{4}{*}{FGVC-Aircraft}             & Val. Acc.                    & 69.5 (0.3)           & 67.7 (0.6)           & 69.4 (0.8)           \\
			                                           & Test Acc.                    & 71.0 (1.0)           & 67.1 (0.4)           & 69.6 (0.7)           \\
			                                           & Val. Acc. w/ Noise           & 36.9 (4.0)           & 28.5 (2.7)           & 48.4 (1.5)           \\
			                                           & Test Acc. w/ Noise           & 39.1 (3.5)           & 30.5 (2.6)           & 49.5 (1.8)           \\
			\midrule
			\multirow{4}{*}{\shortstack[l]{Caltech-UCSD                                                                                                    \\ Birds-200-2011}} & Val. Acc.             & 70.5 (0.5)                      & 70.0 (0.7)                       & 68.9 (0.6)                          \\
			                                           & Test Acc.                    & 67.4 (0.7)           & 66.8 (0.4)           & 66.0 (0.7)           \\
			                                           & Val. Acc. w/ Noise           & 29.7 (1.7)           & 29.3 (2.0)           & 33.4 (1.8)           \\
			                                           & Test Acc. w/ Noise           & 29.0 (1.7)           & 29.2 (2.6)           & 32.2 (1.5)           \\
			\midrule
			\multirow{4}{*}{Stanford Cars}             & Val. Acc.                    & 84.5 (0.4)           & 83.9 (0.4)           & 83.7 (0.5)           \\
			                                           & Test Acc.                    & 83.3 (0.2)           & 81.9 (0.8)           & 82.1 (0.6)           \\
			                                           & Val. Acc. w/ Noise           & 57.5 (1.2)           & 55.2 (0.9)           & 58.2 (1.2)           \\
			                                           & Test Acc. w/ Noise           & 56.0 (0.9)           & 54.5 (0.7)           & 56.4 (1.2)           \\
			\bottomrule
		\end{tabular}
	\end{center}
\end{table}

\tabref{tab:aa} summarizes the results for ResNet-AA, which adopts anti-aliasing average pooling architecture \citep{DBLP:conf/icml/Zhang19}. Specifically, ResNet-AA adopts average pooling in all downsampling layers as well as replacing the max pooling in the stem with average pooling. ResNet-AA was marginally more robust than the ResNet with average pooling only in the stem, but not as significant as the difference with the original ResNet. The result indicates that the core difference in robustness was caused by the use of average pooling in the stem.

\begin{table}[t!]
	\caption{Results on ResNet-AA}
	\label{tab:aa}
	\begin{center}
		\begin{tabular}{ll|ccc}
			\toprule
			\textbf{Dataset                          } & \textbf{Model              } & \textbf{ResNet-AA-50 } & \textbf{ResNet-AA-50-D } & \textbf{ResNet-AA-101-D } \\
			\midrule
			\multirow{4}{*}{Oxford-IIIT Pet}           & Val. Acc.                    & 84.8 (0.8)             & 86.9 (0.4)               & 86.2 (0.3)                \\
			                                           & Test Acc.                    & 83.1 (0.8)             & 84.7 (0.8)               & 84.3 (0.2)                \\
			                                           & Val. Acc. w/ Noise           & 50.1 (2.7)             & 55.6 (2.0)               & 58.1 (2.7)                \\
			                                           & Test Acc. w/ Noise           & 49.6 (3.2)             & 53.9 (1.4)               & 58.9 (2.9)                \\
			\midrule
			\multirow{4}{*}{Caltech-101}               & Val. Acc.                    & 80.2 (0.4)             & 81.7 (0.7)               & 83.0 (0.3)                \\
			                                           & Test Acc.                    & 79.5 (0.6)             & 80.6 (0.5)               & 80.9 (0.5)                \\
			                                           & Val. Acc. w/ Noise           & 61.2 (1.6)             & 61.7 (2.2)               & 65.0 (1.5)                \\
			                                           & Test Acc. w/ Noise           & 60.1 (1.5)             & 60.8 (2.8)               & 63.3 (1.3)                \\
			\midrule
			\multirow{4}{*}{FGVC-Aircraft}             & Val. Acc.                    & 67.3 (0.5)             & 69.8 (1.0)               & 69.1 (0.6)                \\
			                                           & Test Acc.                    & 67.1 (0.9)             & 70.7 (1.2)               & 70.0 (0.9)                \\
			                                           & Val. Acc. w/ Noise           & 40.4 (3.6)             & 45.5 (2.5)               & 49.0 (2.9)                \\
			                                           & Test Acc. w/ Noise           & 42.3 (3.9)             & 48.3 (2.2)               & 49.5 (2.5)                \\
			\midrule
			\multirow{4}{*}{\shortstack[l]{Caltech-UCSD                                                                                                               \\ Birds-200-2011}} & Val. Acc.             & 65.3 (0.6)                      & 68.9 (0.8)                        & 69.4 (0.6)                          \\
			                                           & Test Acc.                    & 62.3 (1.1)             & 66.1 (0.6)               & 66.1 (0.4)                \\
			                                           & Val. Acc. w/ Noise           & 28.6 (0.8)             & 32.5 (1.0)               & 31.7 (2.8)                \\
			                                           & Test Acc. w/ Noise           & 27.5 (1.3)             & 31.4 (1.9)               & 31.0 (2.4)                \\
			\midrule
			\multirow{4}{*}{Stanford Cars}             & Val. Acc.                    & 79.9 (0.6)             & 85.9 (0.3)               & 83.5 (0.6)                \\
			                                           & Test Acc.                    & 78.9 (0.6)             & 83.9 (0.4)               & 81.6 (0.8)                \\
			                                           & Val. Acc. w/ Noise           & 51.8 (1.6)             & 60.3 (2.8)               & 57.2 (3.2)                \\
			                                           & Test Acc. w/ Noise           & 50.3 (1.0)             & 58.9 (2.0)               & 56.0 (3.2)                \\
			\bottomrule
		\end{tabular}
	\end{center}
\end{table}

\tabref{tab:clipothers} summarizes ImageNet-1K results for other ViT configurations, including different patch sizes, resolutions, and training recipes.

\begin{table}[t!]
	\caption{ImageNet-1K results for other ViT configurations}
	\label{tab:clipothers}
	\begin{center}
		\resizebox{1.0\textwidth}{!}{
			\begin{tabular}{ll|ccc}
				\toprule
				\textbf{Pretrained Model                                                } & \textbf{Mean-Std           } & \textbf{Top-1 $\rightarrow$ w/ Noise                  } & \textbf{Rank $\rightarrow$ w/ Noise                } & \textbf{RankDiff} \\
				\midrule
				\texttt{vit\_base\_patch16\_384.augreg\_in1k}                             & \texttt{INCEPTION}           & 81.10                     $\rightarrow$ 60.23           & 676                      $\rightarrow$ 524           & -152              \\
				\texttt{vit\_base\_patch16\_384.augreg\_in21k\_ft\_in1k}                  & \texttt{INCEPTION}           & 85.99                     $\rightarrow$ 70.89           & 129                      $\rightarrow$ 208           & +79               \\
				\texttt{vit\_base\_patch16\_clip\_384.laion2b\_ft\_in12k\_in1k}           & \texttt{OPENAI}              & 87.21                     $\rightarrow$ 70.38           & 55                       $\rightarrow$ 227           & +172              \\
				\texttt{vit\_base\_patch16\_clip\_384.openai\_ft\_in1k}                   & \texttt{OPENAI}              & 86.20                     $\rightarrow$ 68.55           & 110                      $\rightarrow$ 285           & +175              \\
				\texttt{vit\_base\_patch16\_clip\_384.openai\_ft\_in12k\_in1k}            & \texttt{OPENAI}              & 87.03                     $\rightarrow$ 69.11           & 61                       $\rightarrow$ 269           & +208              \\
				\texttt{vit\_base\_patch16\_clip\_384.laion2b\_ft\_in1k}                  & \texttt{OPENAI}              & 86.62                     $\rightarrow$ 66.63           & 83                       $\rightarrow$ 348           & +265              \\
				\midrule
				\texttt{vit\_base\_patch32\_224.augreg\_in1k}                             & \texttt{INCEPTION}           & 74.90                     $\rightarrow$ 58.44           & 1075                     $\rightarrow$ 569           & -506              \\
				\texttt{vit\_base\_patch32\_224.sam\_in1k}                                & \texttt{INCEPTION}           & 73.69                     $\rightarrow$ 51.33           & 1101                     $\rightarrow$ 748           & -353              \\
				\texttt{vit\_base\_patch32\_224.augreg\_in21k\_ft\_in1k}                  & \texttt{INCEPTION}           & 80.71                     $\rightarrow$ 65.31           & 719                      $\rightarrow$ 392           & -327              \\
				\texttt{vit\_base\_patch32\_clip\_224.openai\_ft\_in1k}                   & \texttt{OPENAI}              & 81.93                     $\rightarrow$ 63.94           & 591                      $\rightarrow$ 428           & -163              \\
				\texttt{vit\_base\_patch32\_clip\_224.laion2b\_ft\_in1k}                  & \texttt{OPENAI}              & 82.58                     $\rightarrow$ 63.09           & 504                      $\rightarrow$ 450           & -54               \\
				\texttt{vit\_base\_patch32\_clip\_224.laion2b\_ft\_in12k\_in1k}           & \texttt{OPENAI}              & 83.30                     $\rightarrow$ 65.57           & 419                      $\rightarrow$ 386           & -33               \\
				\midrule
				\texttt{vit\_base\_patch32\_384.augreg\_in1k}                             & \texttt{INCEPTION}           & 78.75                     $\rightarrow$ 59.65           & 893                      $\rightarrow$ 539           & -354              \\
				\texttt{vit\_base\_patch32\_384.augreg\_in21k\_ft\_in1k}                  & \texttt{INCEPTION}           & 83.35                     $\rightarrow$ 63.72           & 412                      $\rightarrow$ 437           & +25               \\
				\texttt{vit\_base\_patch32\_clip\_384.openai\_ft\_in12k\_in1k}            & \texttt{OPENAI}              & 85.21                     $\rightarrow$ 68.40           & 191                      $\rightarrow$ 293           & +102              \\
				\texttt{vit\_base\_patch32\_clip\_384.laion2b\_ft\_in12k\_in1k}           & \texttt{OPENAI}              & 85.37                     $\rightarrow$ 65.58           & 180                      $\rightarrow$ 383           & +203              \\
				\bottomrule
			\end{tabular}
		}
	\end{center}
\end{table}

\tabref{tab:clipcal}, \tabref{tab:clipfgvc}, \tabref{tab:clipcub}, and \tabref{tab:clipcar} summarize full results for fine-tuning ViTs on other datasets. When we replaced the \texttt{OPENAI} mean-std constants with the \texttt{INCEPTION} constants, the CLIP ViTs achieved improved robustness.

\begin{table}[t!]
	\caption{Classification accuracy (\%) for fine-tuning ViTs on the Caltech-101.}
	\label{tab:clipcal}
	\begin{center}
		\resizebox{1.0\textwidth}{!}{
			\begin{tabular}{ll|cc}
				\toprule
				\textbf{Pretrained Model                                               } & \textbf{Mean-Std           } & \textbf{Val. Acc. w/ Noise                          }                                                & \textbf{Test Acc. w/ Noise                                  }             \\
				\midrule
				\texttt{vit\_base\_patch16\_clip\_224.openai\_ft\_in12k\_in1k}           & \texttt{OPENAI}              & 93.1                    (0.6)                  $\rightarrow$ 84.1                             (1.1)  & 92.0                     (0.8)                  $\rightarrow$ 81.8 (1.4)  \\
				\texttt{vit\_base\_patch16\_clip\_224.openai\_ft\_in12k\_in1k}           & \texttt{INCEPTION}           & 95.7                    (0.6)                  $\rightarrow$ 90.4                             (0.8)  & 94.5                     (0.7)                  $\rightarrow$ 89.5 (1.2)  \\
				\texttt{vit\_base\_patch16\_clip\_224.openai\_ft\_in12k\_in1k}           & \texttt{IMAGENET}            & 91.6                    (1.2)                  $\rightarrow$ 80.5                             (2.4)  & 90.5                     (0.8)                  $\rightarrow$ 78.5 (2.4)  \\
				\texttt{vit\_base\_patch16\_clip\_224.datacompxl}                        & \texttt{OPENAI}              & 95.3                    (0.8)                  $\rightarrow$ 86.4                             (2.3)  & 94.6                     (0.6)                  $\rightarrow$ 84.8 (2.1)  \\
				\texttt{vit\_base\_patch16\_clip\_224.datacompxl}                        & \texttt{INCEPTION}           & 96.2                    (0.6)                  $\rightarrow$ 91.0                             (1.3)  & 95.7                     (0.9)                  $\rightarrow$ 89.7 (1.4)  \\
				\texttt{vit\_base\_patch16\_clip\_224.datacompxl}                        & \texttt{IMAGENET}            & 94.7                    (0.7)                  $\rightarrow$ 82.5                             (1.9)  & 93.8                     (1.0)                  $\rightarrow$ 80.8 (2.6)  \\
				\texttt{vit\_base\_patch16\_clip\_224.dfn2b}                             & \texttt{OPENAI}              & 90.2                    (11.2)                 $\rightarrow$ 80.1                             (15.0) & 88.9                     (12.3)                 $\rightarrow$ 78.8 (14.8) \\
				\texttt{vit\_base\_patch16\_clip\_224.dfn2b}                             & \texttt{INCEPTION}           & 96.5                    (0.6)                  $\rightarrow$ 91.7                             (1.2)  & 95.9                     (0.5)                  $\rightarrow$ 91.0 (1.8)  \\
				\texttt{vit\_base\_patch16\_clip\_224.dfn2b}                             & \texttt{IMAGENET}            & 93.7                    (3.9)                  $\rightarrow$ 79.9                             (10.5) & 92.4                     (4.6)                  $\rightarrow$ 78.2 (10.7) \\
				\texttt{vit\_base\_patch16\_clip\_224.metaclip\_2pt5b}                   & \texttt{OPENAI}              & 94.9                    (0.7)                  $\rightarrow$ 81.5                             (2.0)  & 94.2                     (0.7)                  $\rightarrow$ 79.5 (2.0)  \\
				\texttt{vit\_base\_patch16\_clip\_224.metaclip\_2pt5b}                   & \texttt{INCEPTION}           & 96.0                    (0.5)                  $\rightarrow$ 89.5                             (2.1)  & 95.0                     (0.8)                  $\rightarrow$ 87.8 (2.8)  \\
				\texttt{vit\_base\_patch16\_clip\_224.metaclip\_2pt5b}                   & \texttt{IMAGENET}            & 93.6                    (1.0)                  $\rightarrow$ 76.3                             (3.2)  & 92.3                     (1.2)                  $\rightarrow$ 74.6 (2.9)  \\
				\texttt{vit\_base\_patch16\_clip\_224.openai}                            & \texttt{OPENAI}              & 92.8                    (0.2)                  $\rightarrow$ 78.9                             (3.1)  & 91.7                     (1.1)                  $\rightarrow$ 76.9 (3.6)  \\
				\texttt{vit\_base\_patch16\_clip\_224.openai}                            & \texttt{INCEPTION}           & 95.4                    (0.3)                  $\rightarrow$ 87.8                             (0.9)  & 95.4                     (0.6)                  $\rightarrow$ 86.9 (0.9)  \\
				\texttt{vit\_base\_patch16\_clip\_224.openai}                            & \texttt{IMAGENET}            & 92.3                    (0.4)                  $\rightarrow$ 80.3                             (1.8)  & 91.8                     (0.7)                  $\rightarrow$ 77.7 (1.9)  \\
				\texttt{vit\_base\_patch16\_clip\_224.laion2b}                           & \texttt{OPENAI}              & 92.3                    (0.9)                  $\rightarrow$ 77.7                             (2.4)  & 91.2                     (0.6)                  $\rightarrow$ 75.6 (1.6)  \\
				\texttt{vit\_base\_patch16\_clip\_224.laion2b}                           & \texttt{INCEPTION}           & 95.3                    (0.6)                  $\rightarrow$ 87.3                             (0.3)  & 94.3                     (0.6)                  $\rightarrow$ 85.8 (0.5)  \\
				\texttt{vit\_base\_patch16\_clip\_224.laion2b}                           & \texttt{IMAGENET}            & 90.1                    (0.8)                  $\rightarrow$ 71.5                             (2.4)  & 89.2                     (0.5)                  $\rightarrow$ 67.6 (2.4)  \\
				\texttt{vit\_base\_patch16\_224.augreg\_in1k}                            & \texttt{OPENAI}              & 94.4                    (0.3)                  $\rightarrow$ 84.8                             (0.9)  & 94.1                     (0.3)                  $\rightarrow$ 85.7 (0.4)  \\
				\texttt{vit\_base\_patch16\_224.augreg\_in1k}                            & \texttt{INCEPTION}           & 94.1                    (0.3)                  $\rightarrow$ 86.0                             (0.5)  & 93.8                     (0.2)                  $\rightarrow$ 86.7 (0.8)  \\
				\texttt{vit\_base\_patch16\_224.augreg\_in1k}                            & \texttt{IMAGENET}            & 94.3                    (0.6)                  $\rightarrow$ 84.7                             (0.6)  & 94.0                     (0.3)                  $\rightarrow$ 85.7 (0.7)  \\
				\texttt{vit\_base\_patch16\_224.augreg\_in21k}                           & \texttt{OPENAI}              & 97.0                    (0.4)                  $\rightarrow$ 95.1                             (0.5)  & 96.3                     (0.4)                  $\rightarrow$ 94.5 (0.7)  \\
				\texttt{vit\_base\_patch16\_224.augreg\_in21k}                           & \texttt{INCEPTION}           & 97.1                    (0.3)                  $\rightarrow$ 95.8                             (0.5)  & 96.6                     (0.2)                  $\rightarrow$ 95.4 (0.3)  \\
				\texttt{vit\_base\_patch16\_224.augreg\_in21k}                           & \texttt{IMAGENET}            & 97.2                    (0.2)                  $\rightarrow$ 95.1                             (0.2)  & 96.6                     (0.5)                  $\rightarrow$ 94.6 (0.5)  \\
				\texttt{vit\_base\_patch16\_224.mae}                                     & \texttt{OPENAI}              & 92.0                    (0.5)                  $\rightarrow$ 76.3                             (0.7)  & 91.6                     (0.8)                  $\rightarrow$ 75.7 (1.2)  \\
				\texttt{vit\_base\_patch16\_224.mae}                                     & \texttt{INCEPTION}           & 91.6                    (0.6)                  $\rightarrow$ 80.8                             (1.2)  & 91.7                     (0.4)                  $\rightarrow$ 79.4 (0.7)  \\
				\texttt{vit\_base\_patch16\_224.mae}                                     & \texttt{IMAGENET}            & 91.7                    (0.5)                  $\rightarrow$ 75.4                             (0.6)  & 91.6                     (0.4)                  $\rightarrow$ 74.5 (1.2)  \\
				\bottomrule
			\end{tabular}
		}
	\end{center}
\end{table}

\begin{table}[t!]
	\caption{Classification accuracy (\%) for fine-tuning ViTs on the FGVC-Aircraft.}
	\label{tab:clipfgvc}
	\begin{center}
		\resizebox{1.0\textwidth}{!}{
			\begin{tabular}{ll|cc}
				\toprule
				\textbf{Pretrained Model                                               } & \textbf{Mean-Std           } & \textbf{Val. Acc. w/ Noise                          }                                                & \textbf{Test Acc. w/ Noise                                  }             \\
				\midrule
				\texttt{vit\_base\_patch16\_clip\_224.openai\_ft\_in12k\_in1k}           & \texttt{OPENAI}              & 62.6                    (1.7)                  $\rightarrow$ 46.6                             (1.9)  & 61.7                     (1.2)                  $\rightarrow$ 47.4 (1.7)  \\
				\texttt{vit\_base\_patch16\_clip\_224.openai\_ft\_in12k\_in1k}           & \texttt{INCEPTION}           & 60.4                    (23.6)                 $\rightarrow$ 50.8                             (20.4) & 59.5                     (23.7)                 $\rightarrow$ 50.9 (21.1) \\
				\texttt{vit\_base\_patch16\_clip\_224.openai\_ft\_in12k\_in1k}           & \texttt{IMAGENET}            & 59.5                    (1.4)                  $\rightarrow$ 44.0                             (1.6)  & 58.2                     (1.2)                  $\rightarrow$ 45.2 (1.2)  \\
				\texttt{vit\_base\_patch16\_clip\_224.datacompxl}                        & \texttt{OPENAI}              & 73.7                    (4.9)                  $\rightarrow$ 50.7                             (7.7)  & 72.2                     (4.0)                  $\rightarrow$ 52.7 (7.1)  \\
				\texttt{vit\_base\_patch16\_clip\_224.datacompxl}                        & \texttt{INCEPTION}           & 80.8                    (1.9)                  $\rightarrow$ 66.3                             (3.7)  & 79.4                     (2.0)                  $\rightarrow$ 66.3 (3.5)  \\
				\texttt{vit\_base\_patch16\_clip\_224.datacompxl}                        & \texttt{IMAGENET}            & 65.9                    (4.1)                  $\rightarrow$ 40.0                             (6.1)  & 65.0                     (3.4)                  $\rightarrow$ 41.5 (5.2)  \\
				\texttt{vit\_base\_patch16\_clip\_224.dfn2b}                             & \texttt{OPENAI}              & 75.4                    (4.9)                  $\rightarrow$ 55.7                             (7.0)  & 75.2                     (5.2)                  $\rightarrow$ 57.3 (8.0)  \\
				\texttt{vit\_base\_patch16\_clip\_224.dfn2b}                             & \texttt{INCEPTION}           & 82.0                    (4.1)                  $\rightarrow$ 70.0                             (8.1)  & 81.7                     (4.3)                  $\rightarrow$ 70.7 (7.6)  \\
				\texttt{vit\_base\_patch16\_clip\_224.dfn2b}                             & \texttt{IMAGENET}            & 72.9                    (6.6)                  $\rightarrow$ 51.3                             (9.4)  & 71.3                     (7.1)                  $\rightarrow$ 52.5 (9.8)  \\
				\texttt{vit\_base\_patch16\_clip\_224.metaclip\_2pt5b}                   & \texttt{OPENAI}              & 68.0                    (2.7)                  $\rightarrow$ 48.4                             (4.1)  & 67.3                     (2.5)                  $\rightarrow$ 49.7 (3.0)  \\
				\texttt{vit\_base\_patch16\_clip\_224.metaclip\_2pt5b}                   & \texttt{INCEPTION}           & 80.5                    (1.6)                  $\rightarrow$ 68.4                             (3.2)  & 79.2                     (2.2)                  $\rightarrow$ 69.5 (3.4)  \\
				\texttt{vit\_base\_patch16\_clip\_224.metaclip\_2pt5b}                   & \texttt{IMAGENET}            & 64.5                    (1.3)                  $\rightarrow$ 40.9                             (2.4)  & 64.2                     (1.4)                  $\rightarrow$ 43.3 (2.3)  \\
				\texttt{vit\_base\_patch16\_clip\_224.openai}                            & \texttt{OPENAI}              & 63.7                    (4.7)                  $\rightarrow$ 47.4                             (5.6)  & 61.9                     (4.3)                  $\rightarrow$ 49.1 (4.4)  \\
				\texttt{vit\_base\_patch16\_clip\_224.openai}                            & \texttt{INCEPTION}           & 74.6                    (3.5)                  $\rightarrow$ 65.4                             (4.5)  & 73.4                     (3.8)                  $\rightarrow$ 66.0 (5.4)  \\
				\texttt{vit\_base\_patch16\_clip\_224.openai}                            & \texttt{IMAGENET}            & 60.3                    (1.6)                  $\rightarrow$ 42.6                             (2.7)  & 59.4                     (1.4)                  $\rightarrow$ 43.4 (2.6)  \\
				\texttt{vit\_base\_patch16\_clip\_224.laion2b}                           & \texttt{OPENAI}              & 59.9                    (1.9)                  $\rightarrow$ 37.7                             (2.4)  & 58.4                     (1.7)                  $\rightarrow$ 38.5 (1.9)  \\
				\texttt{vit\_base\_patch16\_clip\_224.laion2b}                           & \texttt{INCEPTION}           & 69.2                    (4.4)                  $\rightarrow$ 54.3                             (5.5)  & 68.9                     (5.4)                  $\rightarrow$ 55.0 (6.0)  \\
				\texttt{vit\_base\_patch16\_clip\_224.laion2b}                           & \texttt{IMAGENET}            & 58.3                    (1.8)                  $\rightarrow$ 36.0                             (2.3)  & 56.9                     (1.3)                  $\rightarrow$ 37.3 (2.4)  \\
				\texttt{vit\_base\_patch16\_224.augreg\_in1k}                            & \texttt{OPENAI}              & 67.8                    (0.8)                  $\rightarrow$ 50.7                             (1.9)  & 67.0                     (1.2)                  $\rightarrow$ 51.2 (1.7)  \\
				\texttt{vit\_base\_patch16\_224.augreg\_in1k}                            & \texttt{INCEPTION}           & 67.0                    (0.5)                  $\rightarrow$ 52.4                             (1.4)  & 67.2                     (0.9)                  $\rightarrow$ 53.6 (1.0)  \\
				\texttt{vit\_base\_patch16\_224.augreg\_in1k}                            & \texttt{IMAGENET}            & 67.4                    (0.4)                  $\rightarrow$ 50.1                             (1.4)  & 67.3                     (0.8)                  $\rightarrow$ 51.0 (2.5)  \\
				\texttt{vit\_base\_patch16\_224.augreg\_in21k}                           & \texttt{OPENAI}              & 78.2                    (0.3)                  $\rightarrow$ 69.9                             (0.5)  & 77.2                     (0.6)                  $\rightarrow$ 69.4 (1.1)  \\
				\texttt{vit\_base\_patch16\_224.augreg\_in21k}                           & \texttt{INCEPTION}           & 78.6                    (0.6)                  $\rightarrow$ 71.6                             (0.4)  & 77.3                     (0.4)                  $\rightarrow$ 71.0 (0.4)  \\
				\texttt{vit\_base\_patch16\_224.augreg\_in21k}                           & \texttt{IMAGENET}            & 77.8                    (0.6)                  $\rightarrow$ 68.9                             (0.9)  & 77.1                     (1.0)                  $\rightarrow$ 68.5 (1.2)  \\
				\texttt{vit\_base\_patch16\_224.mae}                                     & \texttt{OPENAI}              & 69.3                    (0.7)                  $\rightarrow$ 39.9                             (4.2)  & 68.8                     (1.5)                  $\rightarrow$ 40.3 (4.4)  \\
				\texttt{vit\_base\_patch16\_224.mae}                                     & \texttt{INCEPTION}           & 69.1                    (0.7)                  $\rightarrow$ 43.5                             (2.8)  & 69.1                     (0.9)                  $\rightarrow$ 44.0 (2.3)  \\
				\texttt{vit\_base\_patch16\_224.mae}                                     & \texttt{IMAGENET}            & 69.1                    (0.6)                  $\rightarrow$ 40.0                             (2.1)  & 69.4                     (1.2)                  $\rightarrow$ 41.8 (1.2)  \\
				\bottomrule
			\end{tabular}
		}
	\end{center}
\end{table}

\begin{table}[t!]
	\caption{Classification accuracy (\%) for fine-tuning ViTs on the Caltech-UCSD Birds-200-2011.}
	\label{tab:clipcub}
	\begin{center}
		\resizebox{1.0\textwidth}{!}{
			\begin{tabular}{ll|cc}
				\toprule
				\textbf{Pretrained Model                                               } & \textbf{Mean-Std           } & \textbf{Val. Acc. w/ Noise                          }                                                & \textbf{Test Acc. w/ Noise                                  }             \\
				\midrule
				\texttt{vit\_base\_patch16\_clip\_224.openai\_ft\_in12k\_in1k}           & \texttt{OPENAI}              & 84.0                    (0.9)                  $\rightarrow$ 64.0                             (1.7)  & 81.3                     (1.0)                  $\rightarrow$ 61.1 (1.1)  \\
				\texttt{vit\_base\_patch16\_clip\_224.openai\_ft\_in12k\_in1k}           & \texttt{INCEPTION}           & 85.3                    (1.6)                  $\rightarrow$ 69.3                             (1.7)  & 82.7                     (1.3)                  $\rightarrow$ 67.0 (2.5)  \\
				\texttt{vit\_base\_patch16\_clip\_224.openai\_ft\_in12k\_in1k}           & \texttt{IMAGENET}            & 82.6                    (0.8)                  $\rightarrow$ 59.8                             (1.3)  & 79.7                     (1.6)                  $\rightarrow$ 56.7 (1.8)  \\
				\texttt{vit\_base\_patch16\_clip\_224.datacompxl}                        & \texttt{OPENAI}              & 83.4                    (1.1)                  $\rightarrow$ 53.6                             (2.3)  & 81.4                     (1.0)                  $\rightarrow$ 50.7 (2.6)  \\
				\texttt{vit\_base\_patch16\_clip\_224.datacompxl}                        & \texttt{INCEPTION}           & 84.7                    (0.7)                  $\rightarrow$ 59.7                             (4.5)  & 82.8                     (0.8)                  $\rightarrow$ 57.3 (3.8)  \\
				\texttt{vit\_base\_patch16\_clip\_224.datacompxl}                        & \texttt{IMAGENET}            & 83.6                    (0.9)                  $\rightarrow$ 52.2                             (2.8)  & 81.5                     (1.1)                  $\rightarrow$ 49.3 (2.6)  \\
				\texttt{vit\_base\_patch16\_clip\_224.dfn2b}                             & \texttt{OPENAI}              & 84.8                    (1.2)                  $\rightarrow$ 58.8                             (2.6)  & 83.0                     (1.3)                  $\rightarrow$ 56.4 (2.3)  \\
				\texttt{vit\_base\_patch16\_clip\_224.dfn2b}                             & \texttt{INCEPTION}           & 87.3                    (1.6)                  $\rightarrow$ 69.6                             (4.7)  & 86.0                     (2.0)                  $\rightarrow$ 67.3 (5.2)  \\
				\texttt{vit\_base\_patch16\_clip\_224.dfn2b}                             & \texttt{IMAGENET}            & 81.6                    (2.7)                  $\rightarrow$ 50.0                             (2.3)  & 79.7                     (2.8)                  $\rightarrow$ 48.1 (2.9)  \\
				\texttt{vit\_base\_patch16\_clip\_224.metaclip\_2pt5b}                   & \texttt{OPENAI}              & 83.3                    (0.5)                  $\rightarrow$ 49.5                             (3.5)  & 81.1                     (0.9)                  $\rightarrow$ 47.9 (3.2)  \\
				\texttt{vit\_base\_patch16\_clip\_224.metaclip\_2pt5b}                   & \texttt{INCEPTION}           & 85.8                    (0.9)                  $\rightarrow$ 62.1                             (2.0)  & 83.4                     (0.6)                  $\rightarrow$ 60.1 (1.8)  \\
				\texttt{vit\_base\_patch16\_clip\_224.metaclip\_2pt5b}                   & \texttt{IMAGENET}            & 81.3                    (2.5)                  $\rightarrow$ 45.3                             (4.5)  & 78.7                     (2.7)                  $\rightarrow$ 43.6 (4.2)  \\
				\texttt{vit\_base\_patch16\_clip\_224.openai}                            & \texttt{OPENAI}              & 83.4                    (0.5)                  $\rightarrow$ 60.1                             (2.4)  & 81.8                     (0.8)                  $\rightarrow$ 57.7 (2.8)  \\
				\texttt{vit\_base\_patch16\_clip\_224.openai}                            & \texttt{INCEPTION}           & 85.5                    (0.8)                  $\rightarrow$ 66.7                             (3.4)  & 83.3                     (1.3)                  $\rightarrow$ 65.1 (3.7)  \\
				\texttt{vit\_base\_patch16\_clip\_224.openai}                            & \texttt{IMAGENET}            & 75.3                    (14.1)                 $\rightarrow$ 50.1                             (13.8) & 72.7                     (13.9)                 $\rightarrow$ 47.5 (12.6) \\
				\texttt{vit\_base\_patch16\_clip\_224.laion2b}                           & \texttt{OPENAI}              & 81.4                    (1.4)                  $\rightarrow$ 52.1                             (2.2)  & 78.5                     (2.5)                  $\rightarrow$ 50.0 (2.0)  \\
				\texttt{vit\_base\_patch16\_clip\_224.laion2b}                           & \texttt{INCEPTION}           & 84.6                    (0.6)                  $\rightarrow$ 62.0                             (2.1)  & 82.2                     (0.4)                  $\rightarrow$ 59.9 (2.1)  \\
				\texttt{vit\_base\_patch16\_clip\_224.laion2b}                           & \texttt{IMAGENET}            & 81.0                    (0.4)                  $\rightarrow$ 50.1                             (0.7)  & 78.7                     (0.5)                  $\rightarrow$ 48.3 (1.1)  \\
				\texttt{vit\_base\_patch16\_224.augreg\_in1k}                            & \texttt{OPENAI}              & 83.4                    (0.4)                  $\rightarrow$ 67.6                             (0.7)  & 81.7                     (0.8)                  $\rightarrow$ 65.8 (0.7)  \\
				\texttt{vit\_base\_patch16\_224.augreg\_in1k}                            & \texttt{INCEPTION}           & 83.9                    (0.5)                  $\rightarrow$ 69.3                             (0.7)  & 81.8                     (0.3)                  $\rightarrow$ 67.8 (0.5)  \\
				\texttt{vit\_base\_patch16\_224.augreg\_in1k}                            & \texttt{IMAGENET}            & 83.7                    (0.4)                  $\rightarrow$ 67.5                             (1.2)  & 81.8                     (0.2)                  $\rightarrow$ 65.9 (0.7)  \\
				\texttt{vit\_base\_patch16\_224.augreg\_in21k}                           & \texttt{OPENAI}              & 89.6                    (0.2)                  $\rightarrow$ 84.0                             (0.5)  & 88.9                     (0.5)                  $\rightarrow$ 83.4 (0.4)  \\
				\texttt{vit\_base\_patch16\_224.augreg\_in21k}                           & \texttt{INCEPTION}           & 89.6                    (0.2)                  $\rightarrow$ 84.9                             (0.7)  & 88.7                     (0.4)                  $\rightarrow$ 83.7 (0.6)  \\
				\texttt{vit\_base\_patch16\_224.augreg\_in21k}                           & \texttt{IMAGENET}            & 89.5                    (0.2)                  $\rightarrow$ 83.9                             (0.1)  & 88.9                     (0.3)                  $\rightarrow$ 83.4 (0.7)  \\
				\texttt{vit\_base\_patch16\_224.mae}                                     & \texttt{OPENAI}              & 76.7                    (0.5)                  $\rightarrow$ 39.3                             (4.4)  & 74.1                     (0.6)                  $\rightarrow$ 36.7 (3.9)  \\
				\texttt{vit\_base\_patch16\_224.mae}                                     & \texttt{INCEPTION}           & 74.0                    (0.3)                  $\rightarrow$ 41.1                             (4.2)  & 72.5                     (1.1)                  $\rightarrow$ 38.9 (4.4)  \\
				\texttt{vit\_base\_patch16\_224.mae}                                     & \texttt{IMAGENET}            & 76.4                    (0.9)                  $\rightarrow$ 38.0                             (1.5)  & 74.4                     (0.5)                  $\rightarrow$ 35.6 (1.7)  \\
				\bottomrule
			\end{tabular}
		}
	\end{center}
\end{table}

\begin{table}[t!]
	\caption{Classification accuracy (\%) for fine-tuning ViTs on the Stanford-Cars.}
	\label{tab:clipcar}
	\begin{center}
		\resizebox{1.0\textwidth}{!}{
			\begin{tabular}{ll|cc}
				\toprule
				\textbf{Pretrained Model                                               } & \textbf{Mean-Std           } & \textbf{Val. Acc. w/ Noise                          }                                               & \textbf{Test Acc. w/ Noise                                  }            \\
				\midrule
				\texttt{vit\_base\_patch16\_clip\_224.openai\_ft\_in12k\_in1k}           & \texttt{OPENAI}              & 83.8                    (0.1)                  $\rightarrow$ 71.0                             (1.4) & 83.0                     (0.7)                  $\rightarrow$ 69.7 (0.7) \\
				\texttt{vit\_base\_patch16\_clip\_224.openai\_ft\_in12k\_in1k}           & \texttt{INCEPTION}           & 87.3                    (1.2)                  $\rightarrow$ 77.7                             (2.1) & 86.2                     (1.3)                  $\rightarrow$ 76.3 (2.2) \\
				\texttt{vit\_base\_patch16\_clip\_224.openai\_ft\_in12k\_in1k}           & \texttt{IMAGENET}            & 81.1                    (1.6)                  $\rightarrow$ 63.8                             (2.2) & 80.7                     (1.9)                  $\rightarrow$ 64.6 (2.2) \\
				\texttt{vit\_base\_patch16\_clip\_224.datacompxl}                        & \texttt{OPENAI}              & 90.1                    (0.7)                  $\rightarrow$ 76.1                             (1.7) & 89.2                     (0.6)                  $\rightarrow$ 75.3 (1.5) \\
				\texttt{vit\_base\_patch16\_clip\_224.datacompxl}                        & \texttt{INCEPTION}           & 91.3                    (0.2)                  $\rightarrow$ 80.9                             (0.8) & 90.4                     (0.6)                  $\rightarrow$ 79.4 (1.0) \\
				\texttt{vit\_base\_patch16\_clip\_224.datacompxl}                        & \texttt{IMAGENET}            & 89.8                    (1.4)                  $\rightarrow$ 75.4                             (3.8) & 89.1                     (1.4)                  $\rightarrow$ 74.3 (3.8) \\
				\texttt{vit\_base\_patch16\_clip\_224.dfn2b}                             & \texttt{OPENAI}              & 91.1                    (0.5)                  $\rightarrow$ 78.9                             (2.5) & 90.2                     (0.5)                  $\rightarrow$ 77.8 (2.2) \\
				\texttt{vit\_base\_patch16\_clip\_224.dfn2b}                             & \texttt{INCEPTION}           & 94.2                    (1.1)                  $\rightarrow$ 88.7                             (2.2) & 93.2                     (1.0)                  $\rightarrow$ 87.6 (2.8) \\
				\texttt{vit\_base\_patch16\_clip\_224.dfn2b}                             & \texttt{IMAGENET}            & 91.1                    (1.8)                  $\rightarrow$ 78.8                             (5.0) & 90.7                     (1.4)                  $\rightarrow$ 77.6 (5.4) \\
				\texttt{vit\_base\_patch16\_clip\_224.metaclip\_2pt5b}                   & \texttt{OPENAI}              & 87.7                    (0.7)                  $\rightarrow$ 67.7                             (1.7) & 86.9                     (0.7)                  $\rightarrow$ 66.4 (1.7) \\
				\texttt{vit\_base\_patch16\_clip\_224.metaclip\_2pt5b}                   & \texttt{INCEPTION}           & 91.1                    (0.3)                  $\rightarrow$ 78.5                             (1.3) & 90.2                     (0.4)                  $\rightarrow$ 77.3 (1.6) \\
				\texttt{vit\_base\_patch16\_clip\_224.metaclip\_2pt5b}                   & \texttt{IMAGENET}            & 87.1                    (1.3)                  $\rightarrow$ 64.7                             (2.0) & 86.1                     (1.7)                  $\rightarrow$ 63.2 (2.3) \\
				\texttt{vit\_base\_patch16\_clip\_224.openai}                            & \texttt{OPENAI}              & 85.6                    (3.3)                  $\rightarrow$ 73.5                             (4.1) & 85.3                     (3.1)                  $\rightarrow$ 72.4 (3.9) \\
				\texttt{vit\_base\_patch16\_clip\_224.openai}                            & \texttt{INCEPTION}           & 89.8                    (0.4)                  $\rightarrow$ 81.0                             (1.1) & 89.5                     (0.4)                  $\rightarrow$ 80.2 (0.7) \\
				\texttt{vit\_base\_patch16\_clip\_224.openai}                            & \texttt{IMAGENET}            & 85.2                    (1.6)                  $\rightarrow$ 70.1                             (3.0) & 84.2                     (1.3)                  $\rightarrow$ 69.0 (3.0) \\
				\texttt{vit\_base\_patch16\_clip\_224.laion2b}                           & \texttt{OPENAI}              & 84.8                    (2.4)                  $\rightarrow$ 65.6                             (4.2) & 84.1                     (2.3)                  $\rightarrow$ 65.3 (3.7) \\
				\texttt{vit\_base\_patch16\_clip\_224.laion2b}                           & \texttt{INCEPTION}           & 89.9                    (0.8)                  $\rightarrow$ 78.4                             (2.3) & 88.8                     (0.9)                  $\rightarrow$ 77.0 (2.2) \\
				\texttt{vit\_base\_patch16\_clip\_224.laion2b}                           & \texttt{IMAGENET}            & 79.9                    (4.7)                  $\rightarrow$ 54.5                             (6.6) & 79.5                     (5.1)                  $\rightarrow$ 54.9 (7.6) \\
				\texttt{vit\_base\_patch16\_224.augreg\_in1k}                            & \texttt{OPENAI}              & 82.8                    (0.5)                  $\rightarrow$ 67.4                             (1.0) & 81.6                     (0.4)                  $\rightarrow$ 66.3 (0.9) \\
				\texttt{vit\_base\_patch16\_224.augreg\_in1k}                            & \texttt{INCEPTION}           & 83.2                    (0.6)                  $\rightarrow$ 69.2                             (1.1) & 81.6                     (0.5)                  $\rightarrow$ 67.5 (1.3) \\
				\texttt{vit\_base\_patch16\_224.augreg\_in1k}                            & \texttt{IMAGENET}            & 83.0                    (0.3)                  $\rightarrow$ 66.2                             (1.4) & 81.5                     (0.2)                  $\rightarrow$ 65.1 (1.6) \\
				\texttt{vit\_base\_patch16\_224.augreg\_in21k}                           & \texttt{OPENAI}              & 89.7                    (0.2)                  $\rightarrow$ 82.6                             (0.5) & 88.5                     (0.3)                  $\rightarrow$ 81.4 (0.5) \\
				\texttt{vit\_base\_patch16\_224.augreg\_in21k}                           & \texttt{INCEPTION}           & 89.9                    (0.2)                  $\rightarrow$ 84.2                             (0.4) & 88.3                     (0.3)                  $\rightarrow$ 83.3 (0.7) \\
				\texttt{vit\_base\_patch16\_224.augreg\_in21k}                           & \texttt{IMAGENET}            & 89.9                    (0.5)                  $\rightarrow$ 81.9                             (0.5) & 88.6                     (0.6)                  $\rightarrow$ 81.1 (0.3) \\
				\texttt{vit\_base\_patch16\_224.mae}                                     & \texttt{OPENAI}              & 80.4                    (0.5)                  $\rightarrow$ 61.1                             (1.5) & 78.0                     (0.6)                  $\rightarrow$ 58.5 (0.9) \\
				\texttt{vit\_base\_patch16\_224.mae}                                     & \texttt{INCEPTION}           & 80.3                    (0.3)                  $\rightarrow$ 61.7                             (1.0) & 77.6                     (0.5)                  $\rightarrow$ 59.3 (0.8) \\
				\texttt{vit\_base\_patch16\_224.mae}                                     & \texttt{IMAGENET}            & 80.6                    (0.4)                  $\rightarrow$ 58.1                             (2.2) & 78.3                     (0.3)                  $\rightarrow$ 56.7 (2.4) \\
				\bottomrule
			\end{tabular}
		}
	\end{center}
\end{table}

\section{Extension to Other Noise Models}
\label{app:other-noise-models}
We select Gaussian noise due to its approximation of aggregate perturbations by the central limit theorem and its prevalence in real-world imaging, such as sensor readout and thermal noise. Here, we explain how our main findings---noise attenuation by larger stem kernels and smaller input resolution (Theorems~\ref{thm:lowpassstems}, \ref{thm:resolution}), the pooling comparison (Theorem~\ref{thm:pooling}), and the normalization effect (Theorem~\ref{thm:clip})---extend beyond Gaussian noise.

\paragraph{Setup.} We continue to use $k$ for a filter side length. For pooling windows, we use $w$ for side length and $m=w^2$ for the number of elements.

\paragraph{Poisson noise.} Let $S$ be the filter support with $|S|=m=k^2$. Let $h=\{h_t\}_{t\in S}$ denote the linear filter coefficients on $S$, and $Y_t\sim\mathrm{Poisson}(x_t)$ independent. For a locally constant intensity on the filter support, where $x_t \approx \bar x$ in a smooth patch, we have
\begin{align}
	\Var[\sum_t h_t Y_t] = \sum_t h_t^2 \Var[Y_t] = \sum_t h_t^2 x_t = \bar x \sum_t h_t^2 + \sum_t h_t^2 (x_t-\bar x) \approx \bar x \|h\|_2^2,
\end{align}
so the per-output-pixel variance inherits the $k^{-2}$ and $s^{-2}$ scalings up to the local factor $\bar x$. Applying the Anscombe transform $A(y)=2\sqrt{y+3/8}$ approximately stabilizes the Poisson variance to $\approx 1$, after which Gaussian-based methods are applicable \citep{anscombe1948transformation}.

\paragraph{Salt-and-pepper noise.} Under the symmetric model where each pixel is replaced by either $0$ or $1$ with probability $q$ and a locally constant patch with mean $\bar x$, we have
\[
	\mathbb{E}[\text{avg error}] = q(1/2-\bar x),
	\quad \Var[\text{avg error}]=O(1/m).
\]
Max pooling tends to amplify these impulses. As a robust alternative, median pooling recovers the clean value in constant patches when contamination is lower than $50\%$ and is $1$-Lipschitz with respect to $\ell_\infty$; trimmed means are another option.

\paragraph{Normalization and Lipschitz sensitivity.} The pixel-space Lipschitz bound in Theorem~\ref{thm:clip} does not depend on the specific noise type, so smaller per-channel normalization stds increase the worst-case sensitivity equally for Gaussian and non-Gaussian perturbations.

\section{Are there other factors that cause vulnerabilities of CLIP?}
\label{app:ablation}
We investigated other factors that might possibly address the vulnerability of CLIP. However, the vulnerability of CLIP could not be fully addressed by other factors examined below.

\paragraph{How about swapping pretrained weights with supervised ViT?} Answer: No. Differences in training datasets and losses would lead to different pretrained weights for CLIP ViTs. Assuming that certain dataset or loss properties, or equivalently certain properties of the pretrained weights of CLIP ViTs, lead to vulnerabilities, we performed controlled experiments to swap parts of them with those of supervised ViTs. Specifically, we swapped pretrained weights of each block in \texttt{vit\_base\_patch16\_clip\_224.openai} with those of \texttt{vit\_base\_patch16\_224.augreg2\_in21k\_ft\_in1k} to see which module weights determine the robustness against Gaussian noise (\tabref{tab:swap}). Although swapping pretrained weights partially addressed the vulnerability of CLIP ViTs in certain cases near the last block such as targeting block12, the improvements were not as significant as the approach of replacing mean-std constants. Furthermore, the improvement depended on the specific weight choice in the target block; block12.mlp.fc2.weight improved robustness, whereas block12.norm1.weight did not. When we swapped multiple weights such as block12.\{mlp.fc2, mlp.fc1, norm2\}, the performance rather degraded, which indicates that improvement is not guaranteed.

\begin{table}[t!]
	\caption{Results of swapping pretrained weights in CLIP ViT. The accuracy with Gaussian noise partially improved.}
	\label{tab:swap}
	\begin{center}
		\begin{tabular}{l|cc}
			\toprule
			\textbf{Swap                         } & \textbf{Val. Acc. $\rightarrow$ w/ Noise                                                               } & \textbf{Test Acc. $\rightarrow$ w/ Noise                                                                } \\
			\midrule
			stem                                   & 53.1                    (1.2)                   $\rightarrow$ 36.3                      (0.9)            & 51.3                      (1.5)                 $\rightarrow$ 35.4                       (1.2)            \\
			block1                                 & 81.3                    (19.1)                  $\rightarrow$ 48.6                      (14.8)           & 80.1                      (19.9)                $\rightarrow$ 46.2                       (14.1)           \\
			block2                                 & 41.9                    (3.4)                   $\rightarrow$ 20.1                      (1.6)            & 40.7                      (3.9)                 $\rightarrow$ 18.8                       (1.6)            \\
			block3                                 & 68.6                    (13.6)                  $\rightarrow$ 29.9                      (4.5)            & 67.9                      (13.7)                $\rightarrow$ 28.7                       (4.4)            \\
			block4                                 & 80.2                    (6.8)                   $\rightarrow$ 29.6                      (7.5)            & 79.8                      (6.4)                 $\rightarrow$ 28.6                       (7.9)            \\
			block5                                 & 77.0                    (4.0)                   $\rightarrow$ 30.0                      (4.0)            & 77.3                      (3.2)                 $\rightarrow$ 29.0                       (3.4)            \\
			block6                                 & 84.7                    (0.8)                   $\rightarrow$ 39.9                      (1.9)            & 84.0                      (0.5)                 $\rightarrow$ 37.8                       (2.2)            \\
			block7                                 & 87.8                    (0.4)                   $\rightarrow$ 45.0                      (0.8)            & 86.5                      (0.6)                 $\rightarrow$ 44.9                       (1.2)            \\
			block8                                 & 90.5                    (0.4)                   $\rightarrow$ 49.9                      (2.7)            & 88.4                      (0.6)                 $\rightarrow$ 48.4                       (1.6)            \\
			block9                                 & 90.7                    (0.2)                   $\rightarrow$ 56.5                      (3.9)            & 90.1                      (0.4)                 $\rightarrow$ 54.8                       (2.2)            \\
			block10                                & 91.5                    (0.4)                   $\rightarrow$ 62.3                      (3.3)            & 91.0                      (0.4)                 $\rightarrow$ 60.4                       (3.5)            \\
			block11                                & 91.4                    (0.4)                   $\rightarrow$ 59.6                      (5.1)            & 90.7                      (0.9)                 $\rightarrow$ 58.1                       (6.0)            \\
			block12                                & 91.5                    (0.5)                   $\rightarrow$ 62.3                      (4.9)            & 91.4                      (0.6)                 $\rightarrow$ 60.8                       (4.4)            \\
			head                                   & 82.8                    (7.6)                   $\rightarrow$ 48.4                      (7.5)            & 82.3                      (6.8)                 $\rightarrow$ 48.2                       (6.6)            \\
			\midrule
			Baseline (\texttt{IMAGENET})           & 91.2                    (0.5)                   $\rightarrow$ 58.5                      (4.0)            & 90.7                      (0.8)                 $\rightarrow$ 58.4                       (4.3)            \\
			Ours (\texttt{INCEPTION})              & 92.5                    (0.3)                   $\rightarrow$ 71.7                      (1.0)            & 91.9                      (0.6)                 $\rightarrow$ 70.2                       (1.2)            \\
			\bottomrule
		\end{tabular}
	\end{center}
\end{table}

\begin{table}[t!]
	\caption{Results of swapping specific weights in block12. Swapping multiple weights did not ensure improved robustness.}
	\label{tab:swapmicro}
	\begin{center}
		\begin{tabular}{l|cc}
			\toprule
			\textbf{Swap                         } & \textbf{Val. Acc. $\rightarrow$ w/ Noise                                                               } & \textbf{Test Acc. $\rightarrow$ w/ Noise                                                                } \\
			\midrule
			block12.norm1.weight                   & 91.0                    (1.3)                    $\rightarrow$ 56.8                      (8.2)           & 90.1                     (1.3)                 $\rightarrow$ 56.6                       (8.6)             \\
			block12.norm1.bias                     & 91.5                    (0.9)                    $\rightarrow$ 59.5                      (5.7)           & 90.7                     (1.4)                 $\rightarrow$ 58.1                       (6.1)             \\
			block12.attn.qkv.weight                & 91.0                    (0.5)                    $\rightarrow$ 60.5                      (1.3)           & 90.3                     (0.8)                 $\rightarrow$ 59.0                       (1.4)             \\
			block12.attn.qkv.bias                  & 91.0                    (0.9)                    $\rightarrow$ 58.7                      (3.6)           & 90.0                     (0.9)                 $\rightarrow$ 58.4                       (3.7)             \\
			block12.attn.proj.weight               & 92.1                    (0.6)                    $\rightarrow$ 59.8                      (5.6)           & 91.3                     (0.9)                 $\rightarrow$ 59.8                       (5.8)             \\
			block12.attn.proj.bias                 & 90.9                    (1.0)                    $\rightarrow$ 58.3                      (5.4)           & 90.2                     (1.2)                 $\rightarrow$ 57.8                       (5.0)             \\
			block12.norm2.weight                   & 91.8                    (0.7)                    $\rightarrow$ 62.7                      (4.3)           & 90.7                     (0.9)                 $\rightarrow$ 61.1                       (3.7)             \\
			block12.norm2.bias                     & 91.4                    (0.9)                    $\rightarrow$ 60.8                      (3.4)           & 90.8                     (0.7)                 $\rightarrow$ 59.7                       (3.4)             \\
			block12.mlp.fc1.weight                 & 91.4                    (1.0)                    $\rightarrow$ 61.0                      (7.6)           & 91.0                     (1.4)                 $\rightarrow$ 60.3                       (7.0)             \\
			block12.mlp.fc1.bias                   & 91.2                    (1.3)                    $\rightarrow$ 58.3                      (3.4)           & 90.4                     (1.5)                 $\rightarrow$ 57.4                       (3.7)             \\
			block12.mlp.fc2.weight                 & 91.3                    (0.4)                    $\rightarrow$ 65.2                      (2.2)           & 90.5                     (0.3)                 $\rightarrow$ 63.8                       (2.1)             \\
			block12.mlp.fc2.bias                   & 91.4                    (0.7)                    $\rightarrow$ 58.8                      (5.0)           & 90.7                     (0.7)                 $\rightarrow$ 58.2                       (5.4)             \\
			\midrule
			block12.mlp.fc2                        & 90.8                    (0.6)                    $\rightarrow$ 58.7                      (3.2)           & 90.0                     (0.4)                 $\rightarrow$ 57.7                       (3.1)             \\
			block12.mlp.fc2 \& mlp.fc1             & 91.9                    (1.1)                    $\rightarrow$ 64.2                      (4.3)           & 91.6                     (0.8)                 $\rightarrow$ 63.8                       (5.0)             \\
			block12.mlp.fc2 \& mlp.fc1 \& norm2    & 91.0                    (0.5)                    $\rightarrow$ 55.5                      (4.3)           & 90.0                     (1.4)                 $\rightarrow$ 54.0                       (5.2)             \\
			\bottomrule
		\end{tabular}
	\end{center}
\end{table}

\paragraph{How about architectural differences such as norm\_pre?} Answer: No. Although the architecture is almost the same for CLIP ViT and supervised ViTs, one difference is that CLIP ViTs insert additional LayerNorm in the patch embedding before the transformer blocks start, which we refer to as norm\_pre. Assuming that the use of norm\_pre causes vulnerability, we performed controlled experiments training ViTs with and without norm\_pre (\tabref{tab:normpre}). Nevertheless, the ViT with norm\_pre, which corresponds to the identical architecture of CLIP ViTs, rather exhibited improved performance against Gaussian noise, which indicates that norm\_pre does not lead to the vulnerability observed in CLIP ViTs.

\begin{table}[t!]
	\caption{Results on different ViT architectures with and without norm\_pre. The use of norm\_pre did not bring vulnerability.}
	\label{tab:normpre}
	\begin{center}
		\begin{tabular}{l|cc}
			\toprule
			\textbf{Architecture  } & \textbf{Top-1 $\rightarrow$ w/ Noise                  } & \textbf{Top-5 $\rightarrow$ w/ Noise                  } \\
			\midrule
			w/o norm\_pre           & 77.76                     $\rightarrow$ 47.15           & 93.84                     $\rightarrow$ 68.96           \\
			w/ norm\_pre            & 78.84                     $\rightarrow$ 54.22           & 94.14                     $\rightarrow$ 76.13           \\
			\bottomrule
		\end{tabular}
	\end{center}
\end{table}

\section{Limitations}
\label{app:lim}
Our study focuses exclusively on robustness to additive Gaussian noise, which, although common in imaging pipelines, does not encompass all real-world corruptions, such as adversarial perturbations, weather effects, or sensor-specific artifacts. Also, the empirical findings are derived from pretrained models in the \texttt{timm} library and controlled experiments on specific datasets, which may represent a limitation in their generalizability to other domains like medical imaging or video processing. Future work could extend these insights to broader corruptions, architectures, and datasets.

\section{Empirical Simulations for Testing Assumption and Theorems}
\label{app:simul}
We performed module-level simulations to compare empirical results with the expected values stated in the assumption and theorems. All simulation results closely matched the theoretical expectations. The used Python source code is available in the supplementary materials.

\paragraph{A\textsubscript{roll}} We embed each $k{\times}k$ kernel into a $512{\times}512$ grid, compute the normalized spectrum $|\widehat K|$, form its $\ell_2$-radial profile, and fit the low-pass envelope $\phi_k(r)=(1+\beta k r)^{-(1+\delta)}$ by weighted log-MSE (\tabref{tab:a}). For representative radii of $\pi/8, \pi/4, \pi/2$, we observed that the empirical magnitudes lie below the fitted envelopes, which verifies this assumption in practice.

\begin{table}[t!]
	\caption{The upper block reports the results for the box kernel. The lower block reports the results for the Gaussian kernel.}
	\label{tab:a}
	\begin{center}
		\begin{tabular}{l|cc}
			\toprule
			\textbf{$r$ (rad) } & \textbf{ Empirical                ($|\widehat K_k(\boldsymbol\omega)|$) } & \textbf{Theoretical ($\phi_k(\|\boldsymbol\omega\|)$) } \\
			\midrule
			0.3962              & 0.0297134                                                                 & 0.0570019                                               \\
			0.7886              & 0.0129235                                                                 & 0.0167515                                               \\
			1.5661              & 0.0040941                                                                 & 0.0042482                                               \\
			\midrule
			0.3962              & 0.0226295                                                                 & 0.0326660                                               \\
			0.7886              & 0.0059380                                                                 & 0.0068950                                               \\
			1.6031              & 0.0007060                                                                 & 0.0010978                                               \\
			\bottomrule
		\end{tabular}
	\end{center}
\end{table}

\paragraph{Theorem~\ref{thm:lowpassstems}} \tabref{tab:t1} reports the Monte Carlo estimate of the per-pixel noise gain $\gamma$ for a $k\times k$ normalized box filter. We convolve i.i.d. $\mathcal{N}(0,\sigma^{2})$ noise with the filter via FFT-based circular convolution and compare the empirical $\hat{\gamma}$ with the theoretical $||K_k||_{F}^{2}=1/k^{2}$, where $K_k$ is the normalized $k\times k$ box stem kernel.

\begin{table}[t!]
	\caption{Measured $\gamma$ for a $k \times k$ kernel. Stds for 100 simulations are reported.}
	\label{tab:t1}
	\begin{center}
		\begin{tabular}{l|cc}
			\toprule
			\textbf{$k$ } & \textbf{Empirical                } & \textbf{Theoretical } \\
			\midrule
			4             & 0.062535 $\pm$  0.001019           & 0.062500              \\
			8             & 0.015584 $\pm$  0.000484           & 0.015625              \\
			12            & 0.006911 $\pm$  0.000296           & 0.006944              \\
			16            & 0.003888 $\pm$  0.000212           & 0.003906              \\
			20            & 0.002487 $\pm$  0.000169           & 0.002500              \\
			24            & 0.001728 $\pm$  0.000148           & 0.001736              \\
			28            & 0.001271 $\pm$  0.000129           & 0.001276              \\
			32            & 0.000973 $\pm$  0.000115           & 0.000977              \\
			\bottomrule
		\end{tabular}
	\end{center}
\end{table}

\paragraph{Theorem~\ref{thm:resolution}} \tabref{tab:t2} reports Monte Carlo estimates of the per-output-pixel noise gain $\gamma_{\downarrow}(s)$ under anti-aliased downsampling by a factor $s$, using a $g(s) \times g(s)$ normalized box prefilter and decimation. We compare the empirical $\hat{\gamma}_{\downarrow}$ with the theoretical $||K_{g(s)}||_F^2=1/g(s)^2$, implying $\sim s^{-2}$ when $g(s)\propto s$.

\begin{table}[t!]
	\caption{Measured $\gamma_{\downarrow}(s)$ for anti-aliased downsampling by a factor of $s$. Stds for 100 simulations are reported.}
	\label{tab:t2}
	\begin{center}
		\begin{tabular}{l|cc}
			\toprule
			\textbf{$s$ } & \textbf{Empirical                } & \textbf{Theoretical } \\
			\midrule
			1             & 0.999667 $\pm$ 0.006170            & 1.000000              \\
			2             & 0.250114 $\pm$ 0.002767            & 0.250000              \\
			3             & 0.110970 $\pm$ 0.001869            & 0.111111              \\
			4             & 0.062447 $\pm$ 0.001483            & 0.062500              \\
			6             & 0.027772 $\pm$ 0.000880            & 0.027778              \\
			8             & 0.015567 $\pm$ 0.000649            & 0.015625              \\
			12            & 0.006945 $\pm$ 0.000433            & 0.006944              \\
			16            & 0.003925 $\pm$ 0.000359            & 0.003906              \\
			\bottomrule
		\end{tabular}
	\end{center}
\end{table}

\paragraph{Theorem~\ref{thm:pooling}} The results in \tabref{tab:t3} were obtained via Monte Carlo with 200k trials on $S+\eta$ with $\eta\sim\mathcal N(0,1)$ and $k=w^2$. Theoretical entries correspond to $\sigma^2/k$ for average pooling and Gauss-Hermite quadrature for $E[M_k]$ and $E[M_k^2]$ to compute max-pooling bias and MSE.

\begin{table}[t!]
	\caption{Comparison of empirical (Em.) and theoretical (Th.) results for average and max poolings}
	\label{tab:t3}
	\begin{center}
		\resizebox{1.0\textwidth}{!}{
			\begin{tabular}{ll|cccccc}
				\toprule
				\textbf{$w$ } & \textbf{$k$ } & \textbf{Avg MSE (Em.) } & \textbf{Avg MSE (Th.) } & \textbf{Max Bias (Em.) } & \textbf{Max Bias (Th.) } & \textbf{Max MSE (Em.) } & \textbf{Max MSE (Th.) } \\
				\midrule
				2             & 4             & 0.25083                 & 0.25000                 & 1.02936                  & 1.02938                  & 1.55372                 & 1.55133                 \\
				3             & 9             & 0.11049                 & 0.11111                 & 1.48535                  & 1.48501                  & 2.56409                 & 2.56262                 \\
				4             & 16            & 0.06265                 & 0.06250                 & 1.76524                  & 1.76599                  & 3.41148                 & 3.41374                 \\
				5             & 25            & 0.04006                 & 0.04000                 & 1.96619                  & 1.96531                  & 4.12369                 & 4.12097                 \\
				6             & 36            & 0.02779                 & 0.02778                 & 2.11722                  & 2.11812                  & 4.71818                 & 4.72069                 \\
				\bottomrule
			\end{tabular}
		}
	\end{center}
\end{table}

\paragraph{Theorem~\ref{thm:clip}} We construct random linear maps \(A\) with \(\|A\|_2=L_z = 3.0\), compose them with \(D=\mathrm{diag}(1/\boldsymbol\sigma)\) from \texttt{INCEPTION} and \texttt{OPENAI}, and estimate \(\|AD\|_2\) via power iteration. \tabref{tab:t4} compares the theoretical bound \(L_z/\sigma_{\min}\) with the measured norm and their ratio, confirming the predicted \(1/\sigma_{\min}\) scaling.

\begin{table}[t!]
	\caption{Measured \(\|AD\|_2\) closely matches the bound \(L_z/\sigma_{\min}\) for random \(A\) under \texttt{INCEPTION} and \texttt{OPENAI}, confirming the \(1/\sigma_{\min}\) scaling}
	\label{tab:t4}
	\begin{center}
		\begin{tabular}{c|ccc}
			\toprule
			\textbf{Constants          } & \textbf{Bound $L_z/\sigma_{\min}$ } & \textbf{Measured $\norm{AD}_2$ } & \textbf{$\dfrac{L_z/\sigma_{\min}}{\norm{AD}_2}$ } \\
			\midrule
			\texttt{INCEPTION}           & 6.000000                            & 5.998213                         & 1.000298                                           \\
			\texttt{OPENAI}              & 11.480943                           & 11.200055                        & 1.025079                                           \\
			\bottomrule
		\end{tabular}
	\end{center}
\end{table}

\section{Rank difference as a robustness proxy}
\label{app:rankdiff}
Here, we denote the rank difference (RankDiff) at severity $\tau>0$,
\[
	\mathrm{RankDiff}_i(\tau) \coloneqq \mathrm{rank}_\tau(i)-\mathrm{rank}_0(i),
\]
where $\mathrm{rank}_\tau$ orders models by accuracy at $\tau$, so a more negative $\mathrm{RankDiff}_i$ indicates a robustness gain. In this section, we show that RankDiff is a principled, scale-free proxy because it aggregates pairwise rank flips caused by robustness slope differences.

\paragraph{Assumption (local linearity with quadratic remainder).} For model $i\in\{1,\dots,M\}$, let $A_i(\tau)$ be its accuracy at noise severity $\tau\ge0$ and $p_i\coloneqq A_i(0)$. For some $\tau_0>0$,
\begin{equation}
	\label{eq:locallin}
	A_i(\tau)=p_i-\rho_i \tau+r_i(\tau),\qquad \rho_i\ge0,\qquad |r_i(\tau)|\le L_i \tau^2\quad(\tau\in[0,\tau_0]),
\end{equation}
where $\rho_i$ is the first-order robustness slope, and $L_i$ bounds the curvature. The linear accuracy drop after applying a specific corruption has been verified in several studies \citep{DBLP:conf/icml/RechtRSS19,DBLP:conf/iclr/HendrycksD19}.

\paragraph{Pairwise flip rule.} For any $i\neq j$,
\begin{equation}
	\label{eq:pair}
	A_i(\tau)-A_j(\tau)=(p_i-p_j)-(\rho_i-\rho_j)\tau+\varepsilon_{ij}(\tau),\qquad |\varepsilon_{ij}(\tau)|\le (L_i+L_j)\tau^2.
\end{equation}
If $\rho_i\neq \rho_j$, the first-order flip threshold is
\begin{equation}
	\label{eq:tstar}
	\tau^\star_{ij} \coloneqq \frac{p_i-p_j}{\rho_i-\rho_j}.
\end{equation}
When $\tau^\star_{ij}\in(0,\tau_0]$ and the margin condition
\begin{equation}
	\label{eq:margin}
	|(p_i-p_j)-(\rho_i-\rho_j)\tau|>(L_i+L_j)\tau^2
\end{equation}
holds at $\tau$, the sign of $A_i(\tau)-A_j(\tau)$ is determined by the first-order term: Model $i$ outranks $j$ at $\tau$ if and only if $\tau>\tau^\star_{ij}$ when $\rho_i<\rho_j$ (\figref{fig:rankflip}).

\begin{figure}[h!]
	\centering
	\includegraphics[width=.99\linewidth]{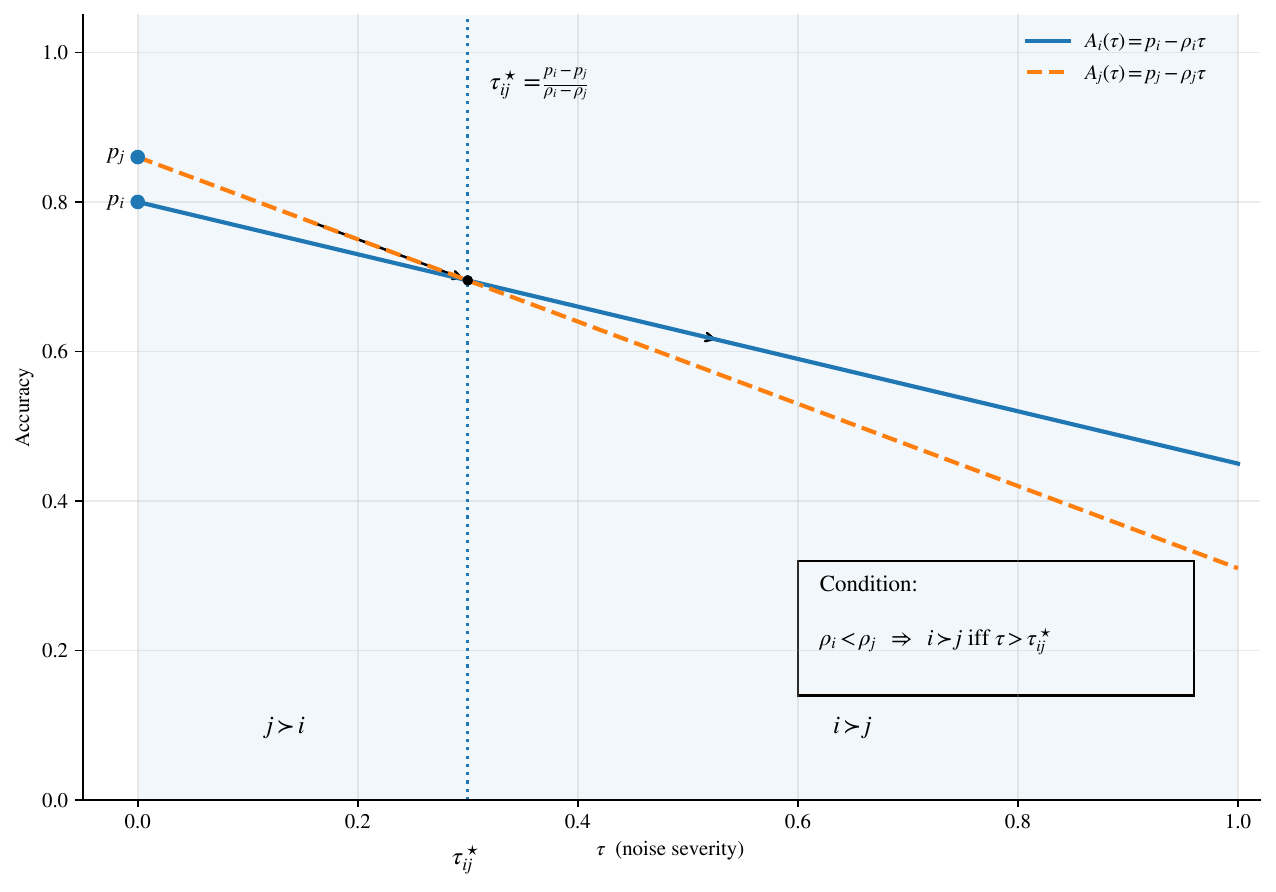}
	\caption{Illustration of a rank flip}
	\label{fig:rankflip}
\end{figure}

\paragraph{RankDiff counts robustness-driven flips.} Let $\mathcal{A}_i(\tau)\coloneqq\{j\neq i:  |(p_i-p_j)-(\rho_i-\rho_j)\tau|\le (L_i+L_j)\tau^2\}$ be the set of ambiguous pairs at $\tau$. Then, we have
\begin{equation}\label{eq:rd-main}
	| \mathrm{RankDiff}_i(\tau) + \sum_{j\neq i} \operatorname{sgn}(\rho_j-\rho_i) \mathbf{1}\{\rho_i\neq \rho_j, 0<\tau^\star_{ij}\le \tau\} | \le |\mathcal{A}_i(\tau)|.
\end{equation}

In particular, if \eqref{eq:margin} holds for all $j\neq i$ at $\tau$, equality holds in \eqref{eq:rd-main}: $\mathrm{RankDiff}_i(\tau)$ equals the net number of pairwise flips caused by having a smaller slope $\rho_i$.

Under the empirically observed near-linearity of accuracy-severity curves within the tested range, RankDiff is a scale-free robustness score: It ignores absolute calibration of accuracies and rewards models with smaller slopes $\rho_i$ by counting the robustness-driven improvements in relative order.

Although we wrote that a negative RankDiff indicates better robustness, we are not saying that understanding the absolute value of RankDiff would capture robustness; to clarify our approach, we rather compare pairwise architectures to compute their corresponding $\Delta \mathrm{RankDiff}$ to understand the relative difference in robustness.

\section{Experimental Setup}
\label{app:extsetup}
Here, we present the experimental details and full hyperparameters for the implementations.

\paragraph{Gaussian Noise} We injected Gaussian noise into images using the \texttt{GaussNoise()} function from the \texttt{Albumentations} library \citep{DBLP:journals/information/BuslaevIKPDK20}. By default, we used the transform \texttt{A.GaussNoise(std\_range=(0.1, 0.22), p=1.0)} with a scale factor with range (0.1, 0.22), which determines the fraction of the maximum value, \ie, 255 for uint8 images or 1.0 for float images. For ImageNet-1K experiments, we used a scale factor with a range of (0.2, 0.44). The probability of applying Gaussian noise was set to 1. Note that Gaussian noise was applied only during evaluation, \ie, during the test phase, not during the training phase.

\paragraph{ResNet Experiments} We targeted multi-class classification tasks on the Oxford-IIIT Pet, Caltech-101, FGVC-Aircraft, Caltech-UCSD Birds-200-2011, and Stanford Cars datasets. The Oxford-IIIT Pet dataset contains 7K pet images from 37 classes; the Caltech-101 dataset includes 9K object images from 101 classes with a background category; the FGVC-Aircraft dataset includes 10K aircraft images from 102 classes; the Caltech-UCSD Birds-200-2011 dataset includes 12K bird images from 200 classes; and the Stanford Cars dataset includes 16K car images from 196 classes. These datasets are publicly available on their official websites. Each dataset was split into training, validation, and test sets with a ratio of 70:15:15. Unless specified otherwise, all experiments were conducted at a resolution of $224^2$ using standard data augmentation, including random resized cropping to 256 pixels, random rotations within 15 degrees, color jitter with a factor of 0.4, random horizontal flip with a probability of 0.5, center cropping with 224-pixel windows, and mean-std normalization based on ImageNet statistics.

For training, stochastic gradient descent with a momentum of 0.9, learning rate of 0.01, cosine annealing schedule with 200 iterations \citep{DBLP:conf/iclr/LoshchilovH17}, weight decay of $10^{-2}$, and mini-batch size of 128 were used. These hyperparameters were determined based on the accuracy of the validation set. One exception was made for experiments with larger resolutions ranging from $224^2$ to $896^2$, where we used mini-batch size of 64 to adjust GPU memory, while other hyperparameters are the same. The model with the highest validation accuracy was obtained after 200 training epochs, and we reported accuracy on the validation and test sets. The ResNets were trained from scratch to solely focus on the architectural difference. The training was conducted on a single GPU machine. An average and standard deviation of five runs with different random seeds were reported for each result.

For ResNet, we used five types with the following architectures:
\begin{itemize}
	\item Original ResNet: $7\times7$ stem with a width = 64 with single-layer, strided convolution in downsampling.
	\item ResNet-C: 3-layer $3\times3$ stem with a width = 32 (32, 32, 64), strided convolution in downsampling.
	\item ResNet-D: 3-layer $3\times3$ stem with a width = 32 (32, 32, 64), average pool in downsampling.
	\item ResNet-S: 3-layer $3\times3$ stem with a width = 64 (64, 64, 128), strided convolution in downsampling.
	\item ResNet-T: 3-layer $3\times3$ stem with a width = 32 (24, 48, 64), average pool in downsampling.
\end{itemize}

\paragraph{CLIP Experiments} For the CLIP experiments, we used pretrained weights for both supervised ViTs and CLIP ViTs. When performing fine-tuning experiments, we used a learning rate of 0.001 and a weight decay of $2 \times 10^{-4}$, while keeping all other hyperparameters the same as in the above setup in ResNet.

\paragraph{ImageNet-1K Training} The ImageNet-1K dataset contains 1.28M images for 1,000 classes. We referred to the hyperparameter recipe described in the official documentation and the recipe from DeiT \citep{DBLP:conf/icml/TouvronCDMSJ21}. For training, the AdamW optimizer \citep{DBLP:conf/iclr/LoshchilovH19} with learning rate $5 \times 10^{-4}$, epochs 400, warm-up learning rate $10^{-6}$, cosine annealing schedule \citep{DBLP:conf/iclr/LoshchilovH17}, weight decay 0.05, label smoothing \citep{DBLP:conf/cvpr/SzegedyVISW16} 0.1, RandAugment \citep{DBLP:conf/nips/CubukZS020} of magnitude 9 and noise-std 0.5 with increased severity (rand-m9-mstd0.5-inc1), random erasing \citep{DBLP:conf/aaai/Zhong0KL020} with probability 0.25, Cutmix \citep{DBLP:conf/iccv/YunHCOYC19} 1.0, stochastic depth \citep{DBLP:conf/eccv/HuangSLSW16} 0.1, mini-batch size 128 per GPU, Exponential Moving Average of model weights with decay factor 0.99996, and image resolution $224^2$ were used. The training was performed on a 4$\times$A100 GPU machine, which required two to three days per training.

\paragraph{Mean-Std Constants} Note that pretrained models may have been trained by any of the normalization constants; our choice of mean-std constants was applied on evaluation or fine-tuning of pretrained models. For training our own models, mean-std constants were applied during both the training and test phases. The exact values are as follows:
\begin{verbatim}
OPENAI_CLIP_MEAN = (0.48145466, 0.4578275, 0.40821073)
OPENAI_CLIP_STD = (0.26862954, 0.26130258, 0.27577711)
IMAGENET_INCEPTION_MEAN = (0.5, 0.5, 0.5)
IMAGENET_INCEPTION_STD = (0.5, 0.5, 0.5)
IMAGENET_DEFAULT_MEAN = (0.485, 0.456, 0.406)
IMAGENET_DEFAULT_STD = (0.229, 0.224, 0.225)
\end{verbatim}

\section{List of Notations}
\label{app:notation}

\begin{table}[h!]
	\caption{Kernel and resolution-related notations.}
	\label{tab:notation-core}
	\centering
	\small
	\begin{tabular}{ll}
		\toprule
		\textbf{Symbol}                                                             & \textbf{Description}                                                   \\
		\midrule
		$x \in [0,1]^{C\times H\times W}$                                           & Input image with $C$ channels, height $H$, width $W$.                  \\
		$\eta \sim \mathcal N(0,\sigma^2 I)$                                        & Additive i.i.d. Gaussian noise with per-pixel std $\sigma$.            \\
		$I, I_n$                                                                    & Identity matrix of appropriate size; $I_n\in\mathbb{R}^{n\times n}$.   \\
		$*$                                                                         & 2D discrete convolution.                                               \\
		$\widehat{u}$                                                               & DFT of $u$ on the grid $\Omega$.                                       \\
		$\Omega$                                                                    & DFT grid.                                                              \\
		$\varepsilon$                                                               & Infrared cutoff $\varepsilon=2\pi/\max\{H,W\}$.                        \\
		$K_k \in \mathbb{R}^{k\times k}$                                            & Stem kernel of side length $k$; $\widehat{K}_k$ denotes its DFT.       \\
		$\phi_k(r)=(1+\beta k r)^{-1-\delta}$                                       & Radial low-pass envelope upper-bounding $|\widehat{K}_k(\omega)|$.     \\
		$\beta, \delta$                                                             & Positive envelope constants.                                           \\
		$\gamma(k)=\dfrac{\mathbb{E}\|K_k*\eta\|_2^2}{\sigma^2 HW}$                 & Per-pixel noise gain of the stem; equals $\|K_k\|_F^2$.                \\
		$s\ge 1$                                                                    & Downsampling factor.                                                   \\
		$g(s)$                                                                      & Anti-alias filter size before downsampling; $c_1 s \le g(s)\le c_2 s$. \\
		$c_1, c_2$                                                                  & Absolute positive constants, independent of $s$.                       \\
		$D_s = (\Downarrow_s)\circ K_{g(s)}$                                        & Anti-aliased downsampling: Filter then downsample by $s$.              \\
		$\Downarrow_s$                                                              & Downsampling by a factor $s$ along height and width.                   \\
		$\gamma_{\downarrow}(s)=\dfrac{\mathbb{E}\|D_s\eta\|_2^2}{\sigma^2 HW/s^2}$ & Per-output-pixel noise gain after downsampling.                        \\
		$\mathbb{E}[\cdot], \mathrm{Var}[\cdot]$                                    & Expectation and variance.                                              \\
		$C, C'$                                                                     & Absolute constants independent of $k$ and $s$ in the bounds.           \\
		$\|\cdot\|_2, \|\cdot\|_\infty, \|\cdot\|_F$                                & Euclidean, sup, and Frobenius norms.                                   \\
		\bottomrule
	\end{tabular}
\end{table}

\begin{table}[h!]
	\caption{Pooling and CLIP-related notations.}
	\label{tab:notation-pool-clip}
	\centering
	\small
	\resizebox{\textwidth}{!}{
		\begin{tabular}{ll}
			\toprule
			\textbf{Symbol}                                                       & \textbf{Description}                                                                                     \\
			\midrule
			$w, m=w^2$                                                            & Pooling window side length and number of elements.                                                       \\
			$S=(S_1,\dots,S_m)$                                                   & Clean activations in one pooling window; $S_{(j)}$ denotes the $j$-th order statistic.                   \\
			$X_{\mathrm{avg}}=\frac{1}{m}\sum_{i=1}^m(S_i+\eta_i)$                & Average-pooled noisy activation.                                                                         \\
			$X_{\mathrm{max}}=\max_{1\le i\le m}(S_i+\eta_i)$                     & Max-pooled noisy activation.                                                                             \\
			$S_{\mathrm{avg}}=\frac{1}{m}\sum_i S_i, S_{\mathrm{max}}=\max_i S_i$ & Clean pooled activations.                                                                                \\
			$\delta_{\mathrm{avg}}=X_{\mathrm{avg}}-S_{\mathrm{avg}}$             & Avg-pool error; $\mathbb{E}[\delta_{\mathrm{avg}}]=0$, $\mathrm{Var}[\delta_{\mathrm{avg}}]=\sigma^2/m$. \\
			$\delta_{\mathrm{max}}=X_{\mathrm{max}}-S_{\mathrm{max}}$             & Max-pool error.                                                                                          \\
			$T_{\mathrm{avg}}, T_{\mathrm{max}}$                                  & Pooling maps on a window for average and max.                                                            \\
			$\|T\|_{\ell_2\to\ell_2}$                                             & Lipschitz constant in $\ell_2$; $\|T_{\mathrm{avg}}\|=m^{-1/2}$, $\|T_{\mathrm{max}}\|\le 1$.            \\
			$\Delta=S_{(1)}-S_{(2)}$                                              & Gap between the largest and second-largest clean entries.                                                \\
			$Z_i\stackrel{\mathrm{i.i.d.}}{\sim}\mathcal N(0,1)$                  & Standard normals; $M_m=\max_i Z_i$, $A_m=\max_i |Z_i|$.                                                  \\
			$\mu\in\mathbb{R}^C, \boldsymbol\sigma\in\mathbb{R}_{>0}^C$           & Per-channel mean and std for input normalization.                                                        \\
			$N_{\mu,\boldsymbol\sigma}(x)=(x-\mu)/\boldsymbol\sigma$              & Channel-wise normalization.                                                                              \\
			$f$                                                                   & Vision backbone operating on normalized inputs.                                                          \\
			$z=N_{\mu,\boldsymbol\sigma}(x)$                                      & Normalized input.                                                                                        \\
			$L_z$                                                                 & Global $\ell_2$-Lipschitz constant of $f$ on its domain.                                                 \\
			$F_{\mu,\boldsymbol\sigma}=f\circ N_{\mu,\boldsymbol\sigma}$          & End-to-end map; $\|F_{\mu,\boldsymbol\sigma}\|_{\mathrm{Lip}}\le L_z/\sigma_{\min}$.                     \\
			$\sigma_{\min}=\min_c \boldsymbol\sigma_c$                            & Smallest channel std in normalization.                                                                   \\
			\bottomrule
		\end{tabular}}
\end{table}

\begin{table}[h!]
	\caption{Rank difference-related notations.}
	\label{tab:notation-ranking}
	\centering
	\small
	\resizebox{\textwidth}{!}{
		\begin{tabular}{ll}
			\toprule
			\textbf{Symbol}                                                      & \textbf{Description}                                                           \\
			\midrule
			$\tau\ge 0$                                                          & Noise severity level.                                                          \\
			$A_i(\tau)$                                                          & Accuracy of model $i$ at severity $\tau$; $p_i=A_i(0)$ denotes clean accuracy. \\
			$\rho_i$                                                             & First-order accuracy slope with respect to severity.                           \\
			$L_i, \tau_0$                                                        & Curvature bound and validity radius for the local model.                       \\
			$\mathrm{rank}_\tau(i)$                                              & Rank of model $i$ by accuracy at severity $\tau$.                              \\
			$\mathrm{RankDiff}_i(\tau)=\mathrm{rank}_\tau(i)-\mathrm{rank}_0(i)$ & Rank change.                                                                   \\
			$\tau^\star_{ij}=\dfrac{p_i-p_j}{\rho_i-\rho_j}$                     & First-order crossing severity of models $i$ and $j$.                           \\
			$\mathcal{A}_i(\tau)$                                                & Set of $j$ whose ordering with $i$ is ambiguous at $\tau$.                     \\
			$\operatorname{sgn}(\cdot), \mathbf{1}\{\cdot\}$                     & Sign and indicator functions.                                                  \\
			\bottomrule
		\end{tabular}}
\end{table}

\end{document}